\documentclass[preprint,1p,12pt]{elsarticle}

\usepackage{xcolor,tikz}
\usetikzlibrary{datavisualization}

\usepackage{hyperref,soul}
\usepackage{amssymb,amsmath}
\usepackage[utf8]{inputenc}
\usepackage{graphicx,color}
\usepackage{cancel}
\usepackage{nicefrac}

\newcommand{\ot}{\leftarrow}

\newcommand{\adjoint}{\mathop{\&}\nolimits}

	\definecolor{naranjauca}{cmyk}{ 0, 0.6, 1, 0}

\usepackage{pgfplots}
\pgfplotsset{compat=newest}

\makeatletter
 \pgfplotsset{
    xtick parsed/.code={
        \c@pgf@counta 0\relax
        \foreach \x in {#1} {
            \pgfmathparse{\x}
            \ifnum\c@pgf@counta=0
                \xdef\pgfplots@xtick{\pgfmathresult}
            \else
                \xdef\pgfplots@xtick{\pgfplots@xtick,\pgfmathresult}
            \fi
            \global\advance\c@pgf@counta 1\relax
        }
    }
 } 
 
 \pgfplotsset{
    ytick parsed/.code={
        \c@pgf@counta 0\relax
        \foreach \x in {#1} {
            \pgfmathparse{\x}
            \ifnum\c@pgf@counta=0
                \xdef\pgfplots@ytick{\pgfmathresult}
            \else
                \xdef\pgfplots@ytick{\pgfplots@xtick,\pgfmathresult}
            \fi
            \global\advance\c@pgf@counta 1\relax
        }
    }
 }

 \makeatother

\journal{Computational and Applied Mathematics}
\date{November 20, 2023}

\newtheorem{theorem}{Theorem}

\newtheorem{lemma}[theorem]{Lemma}
\newtheorem{proposition}[theorem]{Proposition}
\newdefinition{definition}[theorem]{Definition }
\newdefinition{remark}[theorem]{Remark }
\newdefinition{example}[theorem]{Example }
\newproof{proof}{Proof}

\def\upN{\ensuremath{\uparrow_{N}}}
\def\downN{\ensuremath{\downarrow^{N}}}
\def\upPi{\ensuremath{\uparrow_{\pi}}}
\def\downPi{\ensuremath{\downarrow^{\pi}}}
\def\up{\ensuremath{\uparrow}}
\def\down{\ensuremath{\downarrow}}

\begin{document}

\begin{frontmatter}

\title{
Factorizing formal contexts from\\ closures of necessity operators%
\tnoteref{t1}}
\tnotetext[t1]{Partially supported by the 2014--2020 ERDF Operational Programme in collaboration with the State Research Agency (AEI) in projects PID2019-108991GB-I00 and PID2022-137620NB-I00, with the Ecological and Digital Transition Projects 2021 of the Ministry of Science and Innovation in project TED2021-129748B-I00, with the Department of Economy, Knowledge, Business and University of the Regional Government of Andalusia in project FEDER-UCA18-108612, and by the European Cooperation in Science \& Technology (COST) Action CA17124.}

\author[uca]
{Roberto G. Aragón\tnoteref{t2}}
\author[uca]
{Jesús Medina}
\author[uca]
{Eloísa Ramírez-Poussa}

\tnotetext[t2]{Corresponding author.}

\address[uca]
{Department of Mathematics,
 University of  C\'adiz. Spain\\
 Email: \texttt{\{roberto.aragon,jesus.medina,eloisa.ramirez\}@uca.es} }

\begin{abstract}
Factorizing datasets is an interesting process in a multitude of approaches, but many times it is not possible or efficient the computation of a factorization of the dataset. 
A method to obtain independent subcontexts of a formal context with Boolean data was proposed in~\cite{dubois:2012}, based on  the   operators used in possibility theory. In this paper, we will analyze this method and study different properties  related to the pairs of sets from which a factorization of a formal context arises.  We also inspect how the properties given in the classical case can be extended to the fuzzy framework, which is  essential  to obtain a mechanism that allows the computation of independent subcontexts of a fuzzy context.

\end{abstract}


\end{frontmatter}

\section{Introduction}

The factorization of relational data represents a relevant research line since the seventies~\cite{BARTL23,Fagin77,Gugliermo23,Heath71,Trevor96,Kome23,KOYDA202370,OLIVEIRA22}. Being able to factorize a (real) dataset makes possible to reduce the complexity of information processing and, therefore, to obtain the solution of the problem to be solved in a more efficient way~\cite{BELOHLAVEK2015,Chi2023}. Moreover, there exist another two fundamental aspects: the extracted factors reveal important information about the whole  data and the factors can be considered as new variables, originally hidden in the data and revealed by factorization. 

A mathematical theory for extracting knowledge from relational databases is Formal Concept Analysis (FCA, for short) introduced at the beginning of the eighties~\cite{GanterW,wille:1982}. In particular, the databases considered in this theory are called contexts and are composed of a set of objects, a set of attributes and a relationship between these sets. The tools provided by FCA can properly manipulate data and extract relevant information from it, representing the information by means of the algebraic structure of a complete lattice~\cite{OJEDAHERNANDEZ23,singh18,singh19,Zhang10057970}
. One of the most versatile fuzzy extensions of this theory is the one given by the multi-adjoint framework~\cite{TFS:2020-acmr,ar:ins:2015,ins2018:cmr,ins-medina, mor-fss-cmpi}. The characteristics of this framework provide it with a greater capacity to model real problems.

One of the most important research lines within FCA is the factorization of formal contexts. In this regard, different procedures have been already developed, such as~\cite{BELOHLAVEK201836,belohlavek2019,Kumar15MCS,dubois:2012,KRIDLO22IS,TRNECKA201875}.
In~\cite{dubois:2012}, the authors make use of tools from possibility theory to obtain independent subcontexts of a formal context. Specifically, by using modal operators a context is split into two or more independent subcontexts of smaller size which can be studied separately more easily and from which the original context can be recovered.

Recently, in~\cite{aragonCCIS22} an initial study    was presented about the necessity operators applied to the factorization of formal contexts in FCA. 
This study  continues in this paper analyzing the pairs of closed subsets under the necessity operator which form a complemented complete lattice~\cite{georgescu:2004}.  
The minimal independent contexts will be characterized by the supremum irreducible elements of this lattice and these pairs will also satisfy interesting relationships with   property-oriented concepts and   formal concepts. 
In this paper, we continue with the study initiated in~\cite{aragonCCIS22}, providing additional properties that are satisfied by blocks of concepts corresponding to the independent subcontexts in the classical FCA.
Moreover, we will consider the necessity operators in the fuzzy case and analyze the pairs of closures obtained from them. We  will prove that the set of these pairs also forms a complete lattice, but the relationship between these pairs and the independent subcontexts is naturally unclear, in general. Hence, we will study a fuzzy frame in which the translation of the properties of the classical case  will be satisfied.

This paper is organized as follows: Section~\ref{preliminares} recalls different preliminaries notions and results required in this work. In Section~\ref{Facto_c}, properties related to the blocks of concepts associated with independent subcontexts in classical FCA are introduced. The extensions of some of these properties to the fuzzy framework are given in Section~\ref{Facto_f0} together with additional properties to determine intervals of concepts in the fuzzy environment. Finally, we expose some conclusions and prospects of future works in Section~\ref{conclusion}.

\section{Preliminaries}\label{preliminares}

In this section, we recall several notions concerning FCA and its fuzzy generalization to the multi-adjoint framework, which have been taken into consideration to develop this paper.  We will split this section into two parts, in order to expose preliminary notions as clearly as possible. The first part will be focused on classical modal operators and basic notions in FCA and the second part will concern the generalizations of these notions to the multi-adjoint framework.

\subsection{Formal concept analysis}

First of all, a context in FCA is a triple  $(A, B, R)$ where $A$ is a set of attributes, $B$ is a set of objects and $R\subseteq A\times B$ is a relationship, such that $(a,b)\in R$ if the object $b\in B$ possesses the attribute $a\in A$. Moreover, the mappings ${}^\up\colon 2^B \rightarrow 2^A$ and ${}^\down\colon 2^A \rightarrow 2^B$ defined as follow:
\begin{align*}
	X^\uparrow &= \{a\in A\mid (a,b)\in R, \mbox{ for all } b\in X \}\\
	Y^\downarrow &= \{b\in B\mid (a,b)\in R, \mbox{ for all } a\in Y \}
\end{align*}
with $X\subseteq B$ and $Y\subseteq A$, are called \emph{derivation operators}. The pair $(^\up, ^\down)$ forms an antitone Galois connection~\cite{GanterW}. Furthermore, a pair {of sets  $X \subseteq B$ and $Y \subseteq A$, satisfying that $X^{\uparrow}=Y$ and $Y^\downarrow=X$, is called a \emph{concept} and denoted as  $\langle X,Y\rangle$}. In addition, all the concepts together with the inclusion ordering on the left argument (or the opposite of the inclusion order on the right argument) has the algebraic structure of a complete lattice, which is called \emph{concept lattice} and it is denoted by $\mathcal{C}(A,B,R)$, that is, for each $ \langle X_1, Y_1\rangle, \langle X_2, Y_2\rangle\in \mathcal{C}(A,B,R)$, we have that  $ \langle X_1, Y_1\rangle\preceq \langle X_2, Y_2\rangle$ if $X_1\subseteq X_2$ (or, equivalently, $Y_2\subseteq Y_1$).

Besides, the derivation operators previously presented are not the only operators that can be defined on a formal context, there are other modal operators that can be defined  for $X\subseteq B$ and $Y\subseteq A$~\cite{DuboisSP07,DuntschG02,yao06}.  The mappings $^{\upPi}\colon 2^B\rightarrow 2^A$, $^{\downN}\colon 2^A\rightarrow 2^B$, defined as follow:
\begin{align*}
	X^{\upPi} &= \{a\in A\mid \mbox{ there exists } b\in X, \mbox{ such that } (a,b)\in R \}\\
	Y^{\downN} &= \{b\in B\mid \mbox{ if } (a,b)\in R, \mbox{ then }  a\in Y,  \mbox{ for all } a\in A\},
\end{align*}
are called \emph{possibility and necessity operator}, respectively. Analogously, the mappings $^{\upN}\colon 2^B\rightarrow 2^A$, $^{\downPi}\colon 2^A\rightarrow 2^B$ are defined as:
\begin{align*}
	X^{\upN} &= \{a\in A\mid  \mbox{ if } (a,b)\in R, \mbox{ then }  b\in X, \mbox{ for all }  b\in B\}\\
	Y^{\downPi} &= \{b\in B\mid  \mbox{ there exists } a\in Y, \mbox{ such that } (a,b)\in R\}
\end{align*}
The pairs $(^{\upPi}, ^{\downN})$ and $(^{\upN}, ^{\downPi})$ form isotone Galois connections and, as a consequence, the necessity operators preserve the intersection of subsets and the possibility operators preserve unions, that is, $(X_1\cap X_2)^{\upN} = X_1^{\upN} \cap X_2^{\upN}$ and $(X_1\cup X_2)^{\upPi} = X_1^{\upPi} \cup X_2^{\upPi}$, for all $X_1,X_2\subseteq B$, and analogously for attributes~\cite{DisjunctiveINS2021}. From these pairs, we obtain the concept lattices called \emph{property-oriented concept lattice} and \emph{object-oriented concept lattice}~\cite{chenyao08,DuntschG02,camwa-medina}.

In addition,  we need to recall the notion of join-irreducible element of a lattice.

\begin{definition}[\cite{DaveyPriestley}]\label{def:irred}
	Given a lattice $(L,\preceq)$, such that $\vee$ is the join operator,  and an element $x\in L$  verifying:
	\begin{enumerate}
		\item If $L$ has a bottom element $\bot$, then $x\neq \bot$.
		\item If  $x= y\vee z$, then $x=y$ or $x=z$, for all $y,z\in L$.
	\end{enumerate}
	The element $x$ is called \emph{join-irreducible ($\vee$-irreducible) element} of $L$.  
\end{definition}

Another important elements we can find in a lattice are the atoms.

\begin{definition}\label{def:atom}
Let $(L,\preceq)$ be a lattice with bottom element $\bot$. An element $x\in L$ verifying that  $\bot\prec x$ and there is no $y\in L$ such that $\bot \prec y \prec x$ is called an \emph{atom} of $L$. A \emph{coatom} of $L$ is defined dually.
\end{definition}

\subsection{Multi-adjoint framework}\label{version:multi}

As we previously mentioned, in this paper we will consider the fuzzy generalization of FCA given by the multi-adjoint framework~\cite{mor-fss-cmpi}. We recall the generalizations of the previous notions to the multi-adjoint frame.
 First of all, the operators called adjoint triples~\cite{fss:cmr:2013} need to be recalled.

\begin{definition}\label{def:adjoint}
	Let $(P_1,\leq_1)$, $(P_2,\leq_2)$, $(P_3,\leq_3)$ be posets and
	$\adjoint \colon P_1\times P_2 \rightarrow P_3$,
	$\swarrow\colon P_3 \times P_2 \rightarrow P_1$,
	$\nwarrow\colon P_3 \times P_1 \rightarrow P_2$ be
	mappings, then $(\adjoint,\swarrow,\nwarrow)$  is an
	\emph{adjoint triple} with respect to $P_1, P_2, P_3$ if:
	\begin{equation}\label{ap}
		x\leq_1 z  \swarrow y \quad \!\!\hbox{iff} \!\!\quad x \adjoint y \leq_3 z    \quad\!\!  \hbox{iff} \!\!\quad  y\leq_2 z \nwarrow x
	\end{equation}
  	where $x\in P_1$, $y\in P_2$ and $z\in P_3$. Condition~\eqref{ap}  is called \emph{adjoint property}.
\end{definition}

Notice that,  if an operator $\adjoint \colon P_1\times P_2 \rightarrow P_3$ is part of an adjoint triple, then the operators $\swarrow\colon P_3 \times P_2 \rightarrow P_1$,
$\nwarrow\colon P_3 \times P_1 \rightarrow P_2$ are uniquely determined~\cite{ija-cmr15}. 
Moreover, the following proposition can be useful when we have to handle with these operators.

\begin{proposition}[\cite{ija-cmr15}]\label{ijar15}
	Let  $(\&, \swarrow, \nwarrow)$ be an adjoint triple with respect to the posets $(P_1, \leq_1)$, $(P_2, \leq_2)$ and $(P_3, \leq_3)$, then the following properties are satisfied:
	\begin{enumerate}
	\item $\bot_1\& y = \bot_3$, $\top_3\swarrow y = \top_1$, for all $y\in P_2$, when $(P_1, \leq_1, \bot_1,\top_1)$ and $(P_3, \leq_3, \bot_3, \top_3)$ are bounded posets.
	\item $x\&\bot_2 = \bot_3$, $\top_3\nwarrow x = \top_2$, for all $x\in P_1$, when $(P_2, \leq_2, \bot_2,\top_2)$ and $(P_3, \leq_3, \bot_3, \top_3)$ are bounded posets.
	\item When the infimum exists:
	$$
		\left(\bigwedge_{z'\in Z} z'\right)\swarrow y = \bigwedge_{z'\in Z}(z'\swarrow y)
	$$
	for any $Z\subseteq P_3$ and $y\in P_2$.
	\end{enumerate}
\end{proposition}

{
Some examples of adjoint triples are the G\"odel, product and \L ukasiewicz t-norms together with their residuated  implications, which will be used in different examples of this paper. Notice that, the  G\"odel, product and \L ukasiewicz  t-norms are commutative, therefore the residuated implications satisfy that  $\swarrow^{\text{G}}=\nwarrow_{\text{G}}$, $\swarrow^{\text{P}}=\nwarrow_{\text{P}}$ and $\swarrow^{\text{L}}=\nwarrow_{\text{L}}$.  
 \begin{example}\label{ex.non.commut}
Given   $m\in \mathbb N$, the set $[0,1]_m$ is   a regular partition of $[0,1]$ in $m$ pieces, for example  $[0,1]_4=\{0,0.25,0.5,0.75,1\}$ divide the unit interval in four pieces.

A discretization of the product t-norm is the operator  $\adjoint^*_{\mathrm{P}}\colon [0,1]_{4}\times [0,1]_{8}\to [0,1]_{10}$  defined,  for each $x\in
[0,1]_{4}$ and $y\in[0,1]_{8}$ as:
\[
x \adjoint^*_{\mathrm{P}}y= \frac{\textstyle  \lceil 10\cdot x\cdot y\rceil}{\textstyle 10}
\]
 where $\lceil\,\_\,\rceil$ is the ceiling function. In this case the residuated implications
 $\swarrow^*_{\mathrm{P}}  \colon
[0,1]_{10} \times [0,1]_{8}\to [0,1]_{4}$,
$\nwarrow^*_{\mathrm{P}}  \colon [0,1]_{10} \times [0,1]_{4}
\to [0,1]_{8}$ are defined as:
\begin{eqnarray*}
z\swarrow^*_{\mathrm{P}} x  & =  \dfrac { \lfloor 4 \cdot z\ot x\rfloor}{4}\qquad
z\nwarrow^*_{\mathrm{P}} y  & =  \dfrac { \lfloor 8 \cdot  z\ot y\rfloor}{8}  \\
\end{eqnarray*}
where $ z\ot x= 
\begin{cases}
 1& \hbox{if } x\leq y\\
 z/x &\hbox{otherwise}
\end{cases}
$ for all $z,x\in[0,1]$, and    $\lfloor\,\_\,\rfloor$ is the floor function.

Similar adjoint triples can be obtained from  the   Gödel and \L ukasiewicz t-norms~\cite{comparativeAMIS15}. \qed
\end{example}
}

Furthermore, in order to define the notion of context {in the multi-adjoint framework}, we have to fix an algebraic structure called multi-adjoint frame.

\begin{definition}  A \emph{multi-adjoint frame}
is a tuple $(L_1,L_2,P, \adjoint_1, \dots,\adjoint_n )$,
where $(L_1,\preceq_1)$ and $(L_2,\preceq_2)$ are complete lattices, $(P,\leq)$ is a poset and ${(\adjoint_i,\swarrow^i,\nwarrow_i)}$ is an adjoint triple with respect to $L_1, L_2, P$,  for all~$i\in\{1, \dots, n\}$.
\end{definition}

Once a multi-adjoint frame is fixed, the notion of {context in this frame} is defined as follows.

\begin{definition}
 Let  $(L_1,L_2,P, \adjoint_1,
\dots,\adjoint_n  )$ be a multi-adjoint  frame, a  \emph{context} is a tuple $(A,B,R,\sigma)$ such that
 $A$ and $B$ are  non-empty
sets (usually interpreted as attributes and objects, respectively), $R$ is a $P$-fuzzy relation $R \colon A\times B
\rightarrow P$ and $\sigma\colon A\times B\to \{1,\dots,n\}$ is a
mapping  which associates any element in $A\times B$ to some particular
adjoint triple of the frame.

\end{definition}
In addition, the generalization of derivation operators ${}^{\uparrow} \colon L_2^B\rightarrow L_1^A$ and ${}^{\downarrow} \colon L_1^A \rightarrow L_2^B$  are given as follows:
 \begin{eqnarray*}\label{conexGmulti}
	 g^{\uparrow} (a) &=&\inf \{  R(a,b)\swarrow^{\sigma(a,b)}  g(b)\mid b\in B\}\\\label{conexGmulti2}
	 f^{\downarrow} (b)  &=&\inf \{  R(a,b)\nwarrow_{\sigma(a,b)} f(a)\mid a\in A\}
\end{eqnarray*}
for all $g\in L_2^B$, $f\in L_1^A$ and $a\in A$, $b\in B$, where $L_2^B$ and $L_1^A$ denote the set of mappings $g\colon B\to L_2$ and $f\colon A\to L_1$, respectively.

Equivalently, a {\em multi-adjoint concept} is a pair $\langle g, f\rangle$, where $g\in L_2^B $ is a fuzzy subset of objects and $f\in L_1^A $ is a fuzzy subset of attributes, satisfying that $g^{\uparrow} =f$ and $g^{\downarrow} =g$. Furthermore, the set of multi-adjoint concepts together with the usual ordering form a complete lattice.

\begin{definition}
 The  \emph{multi-adjoint concept lattice} associated with a  multi-adjoint frame $(L_1,L_2,P, \adjoint_1, \dots,\adjoint_n  )$ and a context $(A,B,R,\sigma)$ given, is the set
$$
\mathcal{M}=\{\langle g, f\rangle \mid  g\in L_2^B, f\in L_1^A
\hbox{ and } g^{\uparrow} =f, f^{\downarrow}=g\}
$$
where the
ordering is defined by $ \langle g_1, f_1\rangle\preceq \langle g_2, f_2\rangle
\hbox{ if and only if } g_1\preceq_2 g_2 $ (equivalently $f_2\preceq_1
f_1 $), for all $\langle g_1, f_1\rangle, \langle g_2, f_2\rangle\in \mathcal{M}$.
\end{definition}

Similarly to what we have shown for the classical environment, the operators ${}^{\uparrow}$ and ${}^{\downarrow}$ are not the only operators that can be defined on a context.
The classical definitions of the necessity and possibility operators were also generalized in a fuzzy framework~~\cite{ins-medina}, where two complete lattices $(L_1, \preceq_1), (L_2, \preceq_2)$ and a poset $(P, \leq)$ {are fixed}. A multi-adjoint property-oriented frame is given by $(L_1,L_2,P, \adjoint_1,
\dots,\adjoint_n  )$, where $(\adjoint_i,\swarrow^i,\nwarrow_i)$ is an adjoint triple with respect to $P$, $L_2$, $L_1$ for all~$i\in\{1, \dots, n\}$. In this multi-adjoint algebra, the necessity and possibility operators are the mappings ${}^{\downN} \colon L_1^A \rightarrow L_2^B$, ${}^{\upPi} \colon L_2^B\rightarrow L_1^A$, defined as:
\begin{eqnarray*}
	g^{\upPi} (a) &=&\sup \{  R(a,b)\adjoint_{\sigma(a,b)}  g(b)\mid b\in B\}\\
	f^{\downN} (b)  &=&\inf \{  f(a)\nwarrow_{\sigma(a,b)} R(a,b)\mid a\in A\}
\end{eqnarray*}
for all $a\in A, b\in B, g\in L_2^B $ and $f\in L_1^A $. A multi-adjoint object-oriented frame is the tuple $(L_1,L_2,P, \adjoint_1,
\dots,\adjoint_n  )$, where $(\adjoint_i,\swarrow^i,\nwarrow_i)$ is an adjoint triple with respect to $L_1$, $P$, $L_2$ for all~$i\in\{1, \dots, n\}$. In this algebra these operators are given by the mappings ${}^{\upN} \colon L_2^B\rightarrow L_1^A$, ${}^{\downPi} \colon L_1^A \rightarrow L_2^B$ defined as:
\begin{eqnarray*}
	g^{\upN} (a) &=&\inf \{  g(b)\swarrow^{\sigma(a,b)}  R(a,b)\mid b\in B\}\\
	f^{\downPi} (b)  &=&\sup \{  f(a)\adjoint_{\sigma(a,b)} R(a,b)\mid a\in A\}
\end{eqnarray*}
for all $a\in A, b\in B, g\in L_2^B $ and $f\in L_1^A $. 

Notice that the pairs of operators $(^{\upPi}, ^{\downN})$ and $(^{\upN}, ^{\downPi})$ are isotone Galois connections. Moreover, a pair $\langle g, f\rangle$, where $g\in L_2^B $ is a fuzzy subset of objects and $f\in L_1^A $ is a fuzzy subset of attributes, satisfying that $g^{\upPi} =f$ and $f^{\downN} =g$ is called {\em property-oriented concept}. Additionally, a multi-adjoint property-oriented concept lattice is denoted by $\mathcal{M}_{\pi N}$.

{Lastly,  the fuzzy sets $g\in L_2^B $ and $f\in L_1^A $ such that $g(b) = \top_2$, for all $b\in B$, and $f(a)=\top_1$, for all $a\in A$, are denoted as $g_\top$ and $f_\top$, respectively. Similarly, when  $g(b) = \bot_2$, for all $b\in B$, and $f(a)=\bot_1$, for all $a\in A$, they are denoted as $g_\bot$ and $f_\bot$, respectively. }

\section{Properties of the factorization in   FCA }\label{Facto_c}

In~\cite{dubois:2012}, one of the four set-functions of possibility theory is used to characterize independent subcontexts (i.e. contexts that have no common objects and no common properties  from which the original context can be recovered) of a given {finite} context. In particular,  the pair of operators of actual necessity, $(^{\upN}, ^{\downN})$ is considered to decompose the relation $R$ of a context $(A, B, R)$ {with finite sets $A$  and $B$}, by computing the following intersection:
\[
R^* = \bigcap\{ (X\times Y)\cup ({X}^c\times{Y}^c)\mid X\subseteq B, Y\subseteq A,  X^{\upN} = Y,~ Y^{\downN} = X\}
\]
where $X^{c}$ and $Y^{c}$ are the complements of $X$ and $Y$, respectively.

From the study presented in~\cite{dubois:2012}, we can bring to light some interesting properties from a pair of subsets $X\subseteq B$, $Y\subseteq A$, that satisfies the following equalities:
\begin{equation}
	X^{\upN} = Y \mbox{ and } Y^{\downN} = X,\label{op:necessity}
\end{equation}
The set of all pairs satisfying Expression~(\ref{op:necessity}) {is denoted} by $\mathcal{C}_N$, that is, 
\[
\mathcal{C}_N= \{(X, Y)\mid X\subseteq B, Y\subseteq A, X^{\upN} = Y,~ Y^{\downN} = X\}
\]

Different properties of this decomposition will be analyzed in this section, which will be paramount in the extension of
 this factorization procedure to more general environments, such as the fuzzy case provided by the multi-adjoint framework.
In the following, we assume that the data table has neither empty rows nor attributes that are possessed by all objects in the context, this kind of contexts are called \emph{normalized}. Notice that, this assumption is not a real restriction, but if some of the objects/attributes have no attribute/object or have all the attributes/objects, they can be removed at the beginning of the procedure and taken into consideration at the end. In addition, since the context is normalized, the pairs $\langle B, \varnothing \rangle$ and $ \langle \varnothing, A \rangle$ are concepts,  {indeed the are the top and bottom elements of the associated concept lattice, respectively.}

As a starting point, we recall that each pair $(X,Y)$ belonging to $\mathcal{C}_N$ determines an independent  subcontext of the original context~\cite{dubois:2012}. As a consequence, its complement  $(X^{c}, Y^{c})$, where $X^{c}$ and $Y^{c}$ are the complements of $X$ and $Y$ respectively, also belongs to $\mathcal{C}_N$ and thus, determines another independent subcontext.  

\begin{lemma}\label{lemma:complement}
Given a context $(A,B,R)$ and a pair $(X,Y)\in \mathcal{C}_N$, the complement of the pair $(X,Y)$ also belongs to $\mathcal{C}_N$, that is, $(X, Y)^c=(X^{c}, Y^{c})\in \mathcal{C}_N$.
\end{lemma}
\begin{proof}
 We will proceed by reductio ad absurdum. We will suppose that $(X,Y)\in \mathcal{C}_N$ and $(X^c,Y^c)\notin \mathcal{C}_N$. Since $(X^{c})^{\upN}\neq Y^c$, we have to distinguish two different cases:
\begin{itemize}
\item There exists $a\in(X^{c})^{\upN}$ such that $a\notin Y^c$. Therefore, we have that $a\in Y$. But we know that $(X,Y)\in \mathcal{C}_N$, as a consequence, the equality $Y=X^{\upN}$ holds and we have that $a\in X^{\upN}$. Since $a\in(X^{c})^{\upN}$ and $a\in X^{\upN}$, we obtain the following chain 

$$
a\in X^{\upN}\cap(X^{c})^{\upN}=(X\cap X^{c})^{\upN}= \varnothing ^{\upN} = \varnothing
$$
which leads us to a contradiction. 
\item There exists $a\in Y^c$ such that $a\notin(X^{c})^{\upN}$. Since $a\notin(X^{c})^{\upN}$, applying the  definition of the operator ${}^{\upN}$, there exists $b\in B$ such that $(a,b)\in R$ and $b\notin X^c$; that is, $b\in X=Y^{\downN}$ from which it can be concluded that $a\in Y$ which is a contradiction, since $a\notin Y$. 

\end{itemize}
\end{proof}

Therefore, there may be different ways of factorizing the original context into independent subcontexts, depending on the cardinality of the set $\mathcal{C}_N$. 

Since we are interested in reducing the complexity in the data processing as much as possible, we should consider the minimal independent subcontexts. However, pairs in $\mathcal{C}_N$ do not entail any minimality. Therefore, we have to find those pairs that generate minimal subcontexts. In this paper, we propose to obtain these minimal subcontexts from a different approach to the one given in~\cite{dubois:2012}.
The elements of $\mathcal{C}_N$ equipped with the operations
\begin{itemize}
	\item $(X_1, Y_1)\sqcup (X_2, Y_2) = (X_1\cup X_2, Y_1\cup Y_2)$
	\item $(X_1, Y_1)\sqcap (X_2, Y_2) = (X_1\cap X_2, Y_1\cap Y_2)$
	\item $(X, Y)^c=(X^{c}, Y^{c})$
\end{itemize}
where $\sqcup$, $\sqcap$ and ${}^c$ represent the supremum, infimum and complement operators, 
have the structure of {a complemented complete} lattice (this was proved in a more general framework in~\cite{georgescu:2004}), with the inclusion order on the left argument or on the right argument, that is, $(X_1,Y_1)\leq (X_2, Y_2)$ if $X_1\subseteq X_2$ or, equivalently, $Y_1\subseteq Y_2$. Hence, $(\mathcal{C}_N,\leq)$ is a complete lattice~\cite{DaveyPriestley}, where $(B,A)$ and $(\varnothing, \varnothing)$ are the top and bottom elements, respectively. 
From this point of view, we  analyze new properties of the decomposition of a formal context, using the necessity operators. Next, several of these properties will be highlighted and some examples to illustrate the decomposition of a context as well as the introduced properties are given.

The first result we present in this paper relates the elements {in} $\mathcal{C}_N$ that determine minimal subcontexts to the $\vee$-irreducible elements of the complete lattice  associated with the elements in $\mathcal{C}_N$.

\begin{proposition}\label{prop:noless}
Let $(A,B,R)$ be a formal context. An $\vee$-irreducible element of $\mathcal{C}_N$ is an atom of $\mathcal{C}_N$.
\end{proposition}
\begin{proof}
Let us consider a $\vee$-irreducible element $(X^*, Y^*)$ of $\mathcal{C}_N$. We proceed by reduction ad absurdum. We suppose that there exists an element $(X, Y)$ of $\mathcal{C}_N$ such that $(X, Y) \neq (\varnothing, \varnothing)$ and  $(X, Y)<(X^*, Y^*)$. Therefore, the pair $(X^*\backslash X, Y^*\backslash Y)$ is also lesser than $(X^*, Y^*)$. If we prove that this pair belongs to $\mathcal{C}_N$, then we could obtain the $\vee$-irreducible element as union of the pairs $(X, Y)$ and $(X^*\backslash X, Y^*\backslash Y)$ which is a contradiction. Hence, we have the following equalities:
$$
(X^*\backslash X)^{\upN} = (X^*\cap X^c)^{\upN} = (X^*)^{\upN} \cap (X^c)^{\upN} = Y^*\cap Y^c = Y^*\backslash Y
$$
Recall that the second equality holds since the necessity operator satisfies that $ (X_1\cap X_2)^{\upN} = (X_1)^{\upN} \cap (X_2)^{\upN}$. Moreover, by Lemma~\ref{lemma:complement}, we obtain the third equality. The other equality $(Y^*\backslash Y)^{\downN} = X^*\backslash X$ holds in a similar way. Therefore, $(X^*\backslash X, Y^*\backslash Y)\in\mathcal{C}_N$ and consequently, we have that $(X^*, Y^*) = (X, Y) \sqcup (X^*\backslash X, Y^*\backslash Y)$ which is a contradiction. \qed
\end{proof}

From now on, we will denote the $\vee$-irreducible elements of $\mathcal{C}_N$ as $(X^*, Y^*)$. 
In addition, we can go further and determine the smallest independent subcontexts by means of the $\vee$-irreducible elements of $\mathcal{C}_N$.

\begin{proposition}\label{prop:union}
Given a context $(A,B,R)$, the set $\mathcal{C}_N^* $ of all $\vee$-irreducible elements of $(\mathcal{C}_N,\leq)$ determines a partition on the set of attributes $A$ and on the set of objects $B$.
\end{proposition}
\begin{proof}
First of all, given the set $\mathcal{C}_N^*=\{(X^*_i, Y^*_i)\mid i\in I\}$ of all $\vee$-irreducible elements of $\mathcal{C}_N$, where  $I$ is   an index set, we will prove that $\mathcal{C}_N^*$ determines a partition of the set of objects. From the descendent chain condition of the concept lattice~\cite{DaveyPriestley}, which arise because the context is finite, we have that  any element different from the bottom element can be expressed as supremum of  $\vee$-irreducible elements of $\mathcal{C}_N$.  As we previously commented, the supremum in  $(\mathcal{C}_N,\leq)$ is defined by the union, and therefore, we have that the top element of $\mathcal{C}_N$, $(B,A)$, satisfies that $B = \bigcup_{i\in I} X^*_i$. On the other hand, by Proposition~\ref{prop:noless}, we have that every pair $(X^*_i, Y^*_i), (X^*_j, Y^*_j)\in\mathcal{C}_N^*$ with $i\neq j$ satisfies that $X^*_i \cap X^*_j = \varnothing$. Therefore, the set $\{X_i^*\mid i\in I\}$ is a partition of the set of objects $B$. 

Similarly, it is obtained that the set $\{Y_i^*\mid i\in I\}$ is a partition of the set of attributes $A$.\qed

\end{proof}

The following example illustrates the previous results.

\begin{example}\label{ex1}
Let us consider the formal context $(A,B,R)$ {associated with the data} given in {Table~\ref{BoolR}. In addition, the list of concepts and the Hasse diagram of the concept lattice is shown in Figure~\ref{ex:figclas1}.}

	\begin{table}[!h]
		\begin{center}
			\begin{tabular}{|c|cccccc|}
				\hline
				$R$ & $b_1$ & $b_2$ & $b_3$ & $b_4$ & $b_5$ & $b_6$ \\ \hline
				$a_1$& 0 & 1 & 1 & 1 & 0 & 0  \\
				$a_2$& 0 & 0 & 0 & 1 & 0 & 0  \\
				$a_3$& 1 & 0 & 0 & 0 & 0 & 0  \\
				$a_4$& 0 & 0 & 0 & 0 & 1 & 1 \\
				$a_5$& 0 & 0 & 1 & 0 & 0 & 0  \\
				$a_6$& 0 & 0 & 0 & 0 & 1 & 0  \\
				\hline
			\end{tabular}
		\end{center}
		\caption{{Relation  of the formal context}  of Example~\ref{ex1}.}\label{BoolR}
	\end{table}

	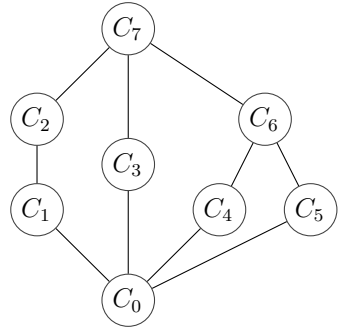
\begin{figure}[h!]
	\begin{center}
	\begin{minipage}{0.45\textwidth}
		\begin{tabular}{l}
			$C_0 = \langle \varnothing, A\rangle $\\
			$C_1 =  \langle \{b_5\}, \{a_4,a_6\}\rangle $\\
			$C_2 =  \langle \{b_5,b_6\}, \{a_4 \}\rangle $\\
			$C_3 =  \langle \{b_1\}, \{a_3 \}\rangle $\\
			$C_4 =  \langle \{b_4\}, \{a_1,a_2 \}\rangle $\\
			$C_5 =  \langle \{b_3\}, \{a_1,a_5 \}\rangle $\\
			$C_6 =  \langle \{b_2,b_3,b_4\}, \{a_1,a_2,a_5\}\rangle $\\
			$C_7 =  \langle B, \varnothing\rangle $
		\end{tabular}
	\end{minipage}
\begin{minipage}{0.45\textwidth}
			\tikzstyle{place}=[circle,draw=black!75,fill=white!20, text width= 14pt]
			\begin{tikzpicture}[inner sep=0.75mm,scale=1.2, every node/.style={scale=0.85}]		
				\node at (0,0) (0) [place] {$C_0$};
				\node at (-1,1) (1) [place] {$C_1$};
				\node at (2,1) (5) [place] {$C_5$};
				\node at (-1,2) (2) [place] {$C_2$};
				\node at (0,1.5) (3) [place] {$C_3$};
				\node at (1,1) (4) [place] {$C_4$};
				\node at (0,3) (7) [place] {$C_{7}$};
				\node at (1.5,2) (6) [place] {$C_{6}$};
				
				\draw [-] (0) -- (1)--(2)--(7)--(3)--(0)--(5)--(6)--(7);
				\draw [-] (0) -- (4)-- (6);
 
			\end{tikzpicture}
	\end{minipage}
	\caption{List of concepts and concept lattice of Example~\ref{ex1}.}
	\label{ex:figclas1}
	\end{center}
	\end{figure}

This context is normalized. Therefore, we can {apply Propositions~\ref{prop:noless} and~\ref{prop:union}}, and compute the elements of the set $\mathcal{C}_N$. The list of elements of $\mathcal{C}_N$ and the complete lattice\footnote{Note that the numbers that appear in the nodes of the complete lattice refer to the subscripts of the corresponding pairs in $\mathcal C_N$. } associated with $(\mathcal{C}_N, \leq)$ are given in Figure~\ref{CN}.

		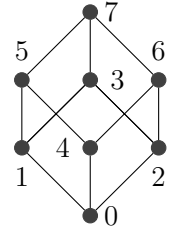
\begin{figure}[!h]
		\begin{center}
	\begin{minipage}{0.7\textwidth}
		\begin{tabular}{l}
			$(X_0,Y_0)= (\varnothing, \varnothing)$\\
			$(X_1,Y_1)= (\{b_5,b_6\},\{a_4,a_6\})$\\
			$(X_2,Y_2) =  (\{b_2,b_3,b_4\}, \{a_1,a_2,a_5\})$\\
			$(X_3,Y_3)=  (\{b_2,b_3,b_4,b_5,b_6\}, \{a_1,a_2,a_4,a_5,a_6 \})$\\
			$(X_4,Y_4)=  (\{b_1\}, \{a_3 \})$\\
			$(X_5,Y_5)=  (\{b_1,b_5,b_6\}, \{a_3,a_4,a_6 \})$\\
			$(X_6,Y_6) =  (\{b_1,b_2,b_3,b_4\}, \{a_1,a_2,a_3,a_5 \})$\\
			$(X_7,Y_7)= (B, A)$\\
		\end{tabular}
	\end{minipage}
\begin{minipage}{0.2\textwidth}
			\tikzstyle{place}=[circle,draw=black!75,fill=black!75]
		\begin{tikzpicture}[inner sep=0.75mm,scale=0.9, every node/.style={scale=0.9}]			
				\node at (0,0) (0) [place, label={[label distance=0.01cm]right:$0$}] {};
				\node at (-1,1) (1) [place, label={[label distance=0.1cm]below:$1$}] {};
				\node at (1,1) (2) [place, label={[label distance=0.1cm]below:$2$}] {};
				\node at (0,2) (3) [place, label={[label distance=0.1cm]right:$3$}] {};
				\node at (0,1) (4) [place, label={[label distance=0.1cm]left:$4$}] {};
				\node at (-1,2) (5) [place, label={[label distance=0.1cm]above:$5$}] {};
				\node at (1,2) (6) [place, label={[label distance=0.1cm]above:$6$}] {};
				\node at (0,3) (7) [place, label={[label distance=0.01cm]right:$7$}] {};
				
				\draw [-] (0) -- (1)--(3)--(2)--(0)--(4)--(5)--(7)--(6)--(4);
				\draw [-] (1) -- (3)-- (2);
				\draw [-] (1) -- (5);
				\draw [-] (2) -- (6);
				\draw [-] (3)-- (7);
 
			\end{tikzpicture}
	\end{minipage}
	\caption{List of elements and lattice of $\mathcal{C}_N$ of Example~\ref{ex1}.}
	\label{CN}
	\end{center}
	\end{figure}
We can see that the $\vee$-irreducible elements are $(X_1,Y_1)$, $(X_2,Y_2)$ and $(X_4,Y_4)$, which  have no concept less than themselves, except the bottom of the lattice, as Proposition~\ref{prop:noless} states. Moreover, we have that $X_1\cup X_2\cup X_4=B$, $X_1\cap X_2=\varnothing$, $X_1\cap X_4=\varnothing$ and $X_2\cap X_4=\varnothing$, as Proposition~\ref{prop:union} asserts. \qed

\end{example}

On the other hand, by their construction, one may think that there is no relation between the pairs of $\mathcal{C}_N$ and the concepts associated with the formal context, but there exists a relevant relationship as it is shown in the following result. Specifically, this result relates the pairs of $\mathcal{C}_N$ to the property-oriented concepts.

\begin{proposition}\label{prop:cp0}
	Given formal context $(A, B, R)$ and a pair $(X, Y)\in\mathcal{C}_N$, satisfying that $X\neq \varnothing$ and $X\neq B$, we have that the pair $\langle X, X^{\upPi}\rangle$ is a concept in the property-oriented framework.
\end{proposition}
\begin{proof}
 The proof straightforwardly follows by  the properties of the Galois connection $(^{\upPi}, ^{\downN})$.
 Specifically, if $(X, Y)\in\mathcal{C}_N$, we obtain that
 $$
 X^{\upPi\downN} =  Y^{\downN\upPi\downN} =  Y^{\downN}  =X
 $$
 \qed
\end{proof}

The following property relates the pairs of $\mathcal{C}_N$ to the formal concepts associated with the context.
\begin{proposition}\label{prop:cp1}
	Let $(A, B, R)$ be a formal context and $(X, Y)\in\mathcal{C}_N$. If $X^\up\neq\varnothing$, then the pair $\langle X, X^\uparrow\rangle $ is a concept of $\mathcal C(A,B,R)$, that is, $X^{\up\down} = X$.
\end{proposition}
\begin{proof}
{If $X=\varnothing$, the result trivially holds since $X^{\up} = \varnothing^{\up} = A$ and $X^{\down\up} = \varnothing^{\down\up}=\varnothing$ because of $(A,B,R)$ is normalized and therefore, $\langle \varnothing, A\rangle $ is a concept. Otherwise, let us consider $X\neq \varnothing$.}
	Since the pair $(^{\up},^{\down})$ is a Galois connection, we have that $X \subseteq X^{\up\down}$. As a consequence, due to $X\neq \varnothing$, then $X^{\up\down}{\neq}\varnothing$.
	
	Now, we consider $b_0\in X^{\up\down} = \{b\in B\mid  (a,b)\in R, \text{ for all } a\in X^{\up}\}$, and we will prove that $b_0\in X$. In particular, taking any {$a_0\in X^\up = \{a\in A \mid (a,b)\in R, \text{ for all } b\in X\}$}, by $b_0\in X^{\up\down} $, we have that $(a_0,b_0)\in R$  and, in addition,    we have that $(a_0,b)\in R$, for all $b\in X$, but $\varnothing\neq X = X^{\upN\downN} = \{b\in B \mid \mbox{ if } (a,b)\in R \mbox{ then } a\in X^{\upN}\}$. Thus, $a_0\in X^{\upN} =\{a\in A \mid \mbox{ if } (a,b)\in R \mbox{ then } b\in X\}$ and therefore, since we have that $(a_0,b_0)\in R$ and $a_0\in X^{\upN}$, we can conclude $b_0\in X$.\qed
\end{proof}

Dually, we can obtain that the pair $\langle Y^\down, Y\rangle$ is a concept of $\mathcal C(A,B,R)$, when the condition  $Y^\down\neq\varnothing$ holds. 

The following result shows when the top element of the independent subcontext differentiates from the top of the original concept lattice.

\begin{proposition}\label{prop:cp2}
	Let $(A, B, R)$ be a formal context and $(X, Y)\in\mathcal{C}_N$, with $X\neq \varnothing$ and $X\neq B$. If $X^\up\neq\varnothing$, then $\langle  X, X^\up\rangle$ is a coatom of $\mathcal{C}(A,B,R)$.
\end{proposition}
\begin{proof}
	We will proceed by reductio ad absurdum. We will suppose that  there exists a concept $\langle X_0, Y_0\rangle$ such that $\langle  X, X^\up\rangle \prec \langle X_0, Y_0\rangle \prec \langle B, \varnothing \rangle$, that is, $X \subset X_0 \subset B$ and $\varnothing \subset Y_0 \subset X^\up$. This means that there exists $b_0\in X_0\setminus X$. Due to $X_0=Y_0^\down$, we have that  $(a,b_0)\in R$, for all $a\in Y_0$. Moreover, due to $Y_0\neq \varnothing$ (because $\varnothing  \subset Y_0 $), we can consider  $a_0\in Y_0$. Since $Y_0 \subset X^\up= \{a\in A \mid (a,b)\in R, \text{ for all } b\in X\}$, we have that  $(a_0, b)\in R$ holds, for all $b\in X$. Now, given $b'\in X$, which exists due to by hypothesis $X\neq\varnothing$,  since 
	\[b'\in X = X^{\upN\downN}=  \{b\in B\mid \mbox{ if } (a,b)\in R \mbox{ then } a\in X^{\upN}\},\]
	we have that $a_0\in X^{\upN} =\{a\in A \mid \mbox{ if } (a,b)\in R \mbox{ then } b\in X\}$. Therefore, from $a_0 \in X^{\upN}$ and $(a_0, b_0)\in R$, we can conclude that $b_0\in X$, which contradicts the hypothesis.\qed
\end{proof}

Consequently, in the previous result it is shown that  the top element of the concept lattice $\mathcal C(A,B,R)$  is the concept directly greater than the concept  $\langle X, X^\up\rangle $ associated with the pair $(X,Y)\in \mathcal C_N$. A similar result can be stated by duality for the pair $\langle Y^\down, Y\rangle $.
Now, the following result analyzes the relationship between $\langle X, X^\up\rangle $  and $\langle Y^\down, Y\rangle $.

\begin{proposition}\label{prop:cp3}
	Given a formal context $(A, B, R)$ and a pair $(X, Y)\in\mathcal{C}_N$ with $X\neq B$, where $X^\up\neq\varnothing$ and $Y^\down\neq\varnothing$, then we have that the concept generated by $\langle X, X^\up\rangle $   is greater than $\langle Y^\down, Y\rangle $, that is, $Y^\down \subseteq X$.
\end{proposition}
\begin{proof}
Given $b_0\in Y^\down = \{b\in B \mid  (a,b)\in R, \text{ for all } a\in Y \}$. Hence, for $a_0 \in Y$ we have that $(a_0,b_0)\in R$ and since $Y = Y^{\downN\upN}=\{a\in A \mid \mbox{ if } (a,b')\in R \mbox{ then } b'\in Y^{\downN}\}$, we have straightforwardly that $b_0\in Y^{\downN} = X$.\qed
	
\end{proof}

Note that the pairs  $(X, X^\uparrow)$ and $(Y^\downarrow, Y)$ do not necessarily have to be concepts from the original concept lattice. This happens when   $X^\up=\varnothing$ (or $Y^\down=\varnothing$) hold. In this case, we  can  find more than one concept that satisfies Proposition~\ref{prop:cp2} and Proposition~\ref{prop:cp3}, that is, there exist maximal elements satisfying these properties (equivalently with the minimal elements, when $Y^\down=\varnothing$ holds). Therefore, when we find several independent subcontexts in this situation, the top elements of the concept lattices associated with these independent subcontexts are identified by the top element of the original concept lattice (similarly,  when the condition $Y^\down=\varnothing$ holds, the bottom elements of the concept lattices are identified by the bottom element of the original concept lattice). 
In the following example, the previously introduced properties are illustrated.

\begin{example}\label{ex:factorclas}
	Let us consider a context $(A, B, R)$ where the set of attributes is given by $A=\{a_1,a_2,a_3,a_4,a_5,a_6,a_7,a_8\}$, the set of objects is $B=\{b_1,b_2,b_3,b_4,b_5,b_6,b_7\}$ and the relationship $R\colon A\times B \rightarrow \{0,1\}$ is given in Table~\ref{R:ej2}.
	
	\begin{table}[!h]
		\begin{center}
			\begin{tabular}{|c|ccccccc|}
				\hline
				$R$ & $b_1$ & $b_2$ & $b_3$ & $b_4$ & $b_5$ & $b_6$ & $b_7$\\ \hline
				$a_1$& 0 & 1 & 1 & 1 & 0 & 0 & 0 \\
				$a_2$& 0 & 0 & 0 & 1 & 0 & 0 & 0 \\
				$a_3$& 1 & 0 & 0 & 0 & 0 & 0 & 0 \\
				$a_4$& 0 & 0 & 0 & 0 & 0 & 1 & 1 \\
				$a_5$& 0 & 1 & 1 & 0 & 0 & 0 & 0 \\
				$a_6$& 0 & 0 & 0 & 0 & 1 & 1 & 0 \\
				$a_7$& 0 & 0 & 1 & 0 & 0 & 0 & 0 \\
				$a_8$& 0 & 0 & 0 & 0 & 1 & 1 & 1 \\
				\hline
			\end{tabular}
		\end{center}
		\caption{Relation of the formal context of Example~\ref{ex:factorclas}.}\label{R:ej2}
	\end{table}

First of all, we have to compute all pairs of $\mathcal{C}_N$. These pairs, except the trivial ones, are the following:
\begin{align*}
	(X_1, Y_1) = & (\{b_5,b_6,b_7\}, \{a_4,a_6,a_8\})\\ 	
	(X_2, Y_2) = & (\{b_1\}, \{a_3\})\\ 
	(X_3, Y_3) = & (\{b_2,b_3,b_4\}, \{a_1,a_2,a_5,a_7\}) \\ 
	(X_4, Y_4) = & (\{b_1,b_5,b_6,b_7\}, \{a_3,a_4,a_6,a_8\})\\
	(X_5, Y_5) = & (\{b_1,b_2,b_3,b_4\}, \{a_1,a_2,a_3,a_5,a_7\}) \\
	(X_6, Y_6) = & (\{b_2,b_3,b_4,b_5,b_6,b_7\}, \{a_1,a_2,a_4,a_5,a_7\}) 
\end{align*}

If we analyze the first pair, $(X_1, Y_1)$, we can see that $X_1^\up = \{a_8\} \neq \varnothing$ and, moreover,  $\{a_8\}^\down= \{b_5,b_6,b_7\}= X$. This means that $X_1^{\up\down} = X_1$, that is, $\langle X_1, X_1^\up\rangle$ is a formal concept as Proposition~\ref{prop:cp1} states. In particular, $\langle X_1, X_1^\up\rangle = C_8$ as it can be seen in the list of concepts shown in Figure~\ref{ex:figclas}, which also shows the concept lattice associated with the context. In this concept lattice can be checked that $C_8$  is less than the concept $C_{10}$, which is the maximum element of the lattice, and there is no intermediate concept between them, that is, the pair $\langle X_1,X_1^\up\rangle$ also satisfies Proposition~\ref{prop:cp2}. 

Now, let us check that the pair $(X_1, Y_1)$ also satisfies Proposition~\ref{prop:cp3}. For that, we need to see that $Y_1^\down \neq \varnothing$. Indeed, $Y_1^\down = \{b_6\}\neq \varnothing$, that is, $\langle Y_1^\down,Y_1\rangle = C_1$ and moreover $Y_1^\down = \{b_6\}\subseteq \{b_5,b_6, b_7\}=X$ as Proposition~\ref{prop:cp3} states. 

As a conclusion, the pair $(X_1, Y_1)$ satisfies the three properties we have presented, that is, the pair $(X_1, Y_1)$ determines one block (interval) of concepts bounded by the concepts $C_8$ and $C_1$, as can be seen in Figure~\ref{ex:figclas}.

\begin{figure}[h!]
	\begin{minipage}{0.35\textwidth}
		\begin{tabular}{l}
			$C_0 = \langle \varnothing, A\rangle $\\
			$C_1 =  \langle \{b_6\}, \{a_4,a_6,a_8\}\rangle $\\
			$C_2 =  \langle \{b_3\}, \{a_1, a_5,a_7 \}\rangle $\\
			$C_3 =  \langle \{b_5,b_6\}, \{a_6,a_8 \}\rangle $\\
			$C_4 =  \langle \{b_6,b_7\}, \{a_4,a_8 \}\rangle $\\
			$C_5 =  \langle \{b_1\}, \{a_3 \}\rangle $\\
			$C_6 =  \langle \{b_4\}, \{a_1,a_2\}\rangle $\\
			$C_7 =  \langle \{b_2,b_3\}, \{a_1,a_5 \}\rangle $\\
			$C_8 =  \langle \{b_5,b_6,b_7\}, \{a_8 \}\rangle $\\
			$C_9 = \langle \{b_2,b_3,b_4\}, \{a_1\}\rangle $\\
			$C_{10} =  \langle B, \varnothing\rangle $
		\end{tabular}
	\end{minipage}
\begin{minipage}{0.6\textwidth}
			\tikzstyle{place}=[circle,draw=black!75,fill=white!20, text width= 14pt]
			\begin{tikzpicture}[inner sep=0.75mm,scale=1.2, every node/.style={scale=0.85}]		
				\node at (0,0) (0) [place] {$C_0$};
				\node at (-2,1) (1) [place] {$C_1$};
				\node at (2.75,1.25) (2) [place] {$C_2$};
				\node at (-3,2) (3) [place] {$C_3$};
				\node at (-1,2) (4) [place] {$C_4$};
				\node at (0,2) (5) [place] {$C_5$};
				\node at (1.25,1.75) (6) [place] {$C_6$};
				\node at (2.75,2.25) (7) [place] {$C_7$};
				\node at (-2,3) (8) [place] {$C_8$};
				\node at (0,4) (9) [place] {$C_{10}$};
				\node at (2,3) (10) [place] {$C_{9}$};
				
				\draw [-] (0) -- (1)-- (3)--(8)--(4)--(1);
				\draw [-] (0) -- (5)-- (9)--(8);
				\draw [-] (0) -- (2)-- (7)--(10)--(6)--(0);
				\draw [-] (9) -- (10);
			\end{tikzpicture}
	\end{minipage}
	\caption{List of concepts and concept lattice of Example~\ref{ex:factorclas}.}
	\label{ex:figclas}
	\end{figure}
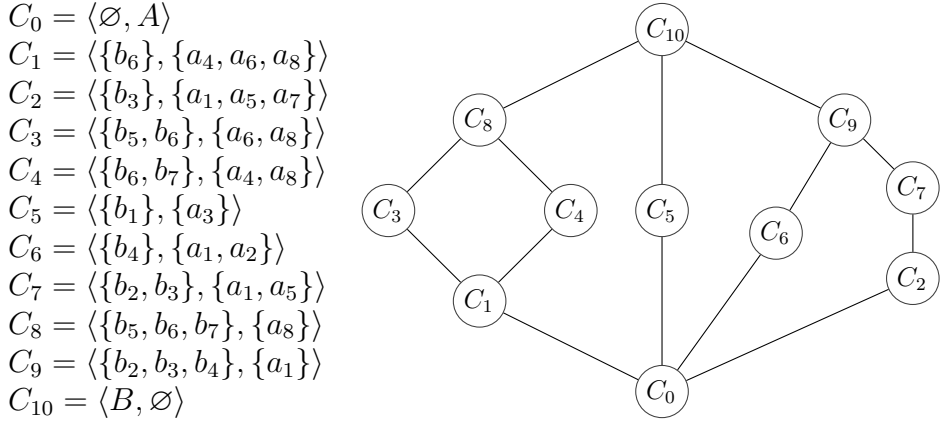
	
On the other hand, if we focus on the pair $(X_3,Y_3)$ and we follow a similar procedure to the one given for the pair $(X_1, Y_1)$, we can check that $\langle X_3,X_3^\up\rangle= \langle\{b_2,b_3,b_4\}, \{a_1\}\rangle = C_9$ and $C_9$ is less than $C_{10}$ and there is no intermediate concept between $C_9$ and $C_{10}$, that is, the pair satisfies Proposition~\ref{prop:cp1} and Proposition~\ref{prop:cp2}. However, if we compute the extent of $Y_3$ we obtain that $Y_3^\down=\varnothing$ and, therefore, $(Y_3^\down, Y_3)$ is not a concept and does not satisfy any of the presented properties. However, the pair $(X_3, Y_3)$ still defines a block of concepts comprised by $C_0$, $C_2$, $C_6$, $C_7$ and $C_9$. In other words, {the fact that the pair $(Y_3^\down, Y_3)$ is not a concept means that the bottom element of the concept lattice associated with the subcontext generated by $(X_3, Y_3)$ is  identified by the bottom element of the original concept lattice $\langle \varnothing, A\rangle$.} \qed
\end{example}

Notice that, in all the previous results, we are requiring that $X^\up\neq \varnothing$, that is,  there exists at least one attribute in the subcontext which is shared by all the objects of the considered subcontext $(X,Y,R_{X\times Y})$ where $R_{X\times Y}$ denotes the restriction of the relation $R$ to the subsets $X$ and $Y$. Equivalently, $Y^\down\neq\varnothing$ indicates the existence of at least one object that possesses all the attributes of the subcontext.

\section{Toward the factorization of contexts in  multi-adjoint frameworks}\label{Facto_f0}

We are interested in generalizing the factorization of formal contexts into independent subcontexts in the fuzzy environment provided by the multi-adjoint framework. With that goal, in this section we study how the properties {and results} presented in the previous section {can be} translated into the multi-adjoint framework. 

In order to simplify the results presented in this section, we will consider a multi-adjoint frame with only one adjoint triple. Therefore, we fix the frame {$(L_1, L_2, P,\&, \swarrow, \nwarrow)$} and  the context $(A, B, R)$, where the mapping $\sigma$  does not appear since it is not necessary when we consider a single adjoint triple. As we commented in the previous section, to find independent subcontexts in the classical case, one of the restrictions imposed on the relation of the context is that it should be normalized. The fuzzy relations of the contexts in this fuzzy environment should also keep the direct extension of this condition to the fuzzy case, i.e., there are no rows or columns with all values equal to bottom or with all values different from bottom. In this environment, we consider the pair $(g,f)$, with $g\in L_2^B$ and $f\in L_1^A$,  satisfying the equalities 
$$g^{\upN} = f \mbox{ and } f^{\downN} =g,$$
but considering  the generalized versions of the necessity operators, given in Section~\ref{version:multi}. Specifically, we will consider the set:
\[\mathcal{F}_N = \{(g,f)\mid g\in L_2^B, f\in L_1^A, f^{\downN} = g, g^{\upN} = f\}\]

 {Note that $\mathcal{F}_N\neq \varnothing$ since it has at least one element as the following lemma states.}
 {
 \begin{lemma}\label{lemma:tops}
 Given a frame  $(L_1, L_2, P,\&, \swarrow, \nwarrow)$ and a context $(A, B,R)$, then the pair $(g_\top, f_\top)$ is an element of the set $\mathcal{F}_N$.
 \end{lemma}
 \begin{proof}
 We have to prove that the pair $(g_\top, f_\top)$ satisfies that  $f_\top^{\downN} = g_\top$ and $g_\top^{\upN} = f_\top$.  We will prove that $g_\top^{\upN} = f_\top$, the other equality is obtained analogously. For every $a\in A$, we have that:
 \begin{align*}
	 g_\top^{\upN} (a) 
	&=\inf\{ g_\top(b)\swarrow R(a,b) \mid b\in B \}\\ 
	&= \inf\{ \top_2 \swarrow R(a,b) \mid b\in B \}\\
	&\overset{(*)}{=} \inf\{ \top_1 \mid b\in B \}\\
	&= \top_1\\
	&= f_\top(a)
\end{align*}
Note that $(*)$ holds by Proposition~\ref{ijar15}. Therefore, $g_\top^{\upN} = f_\top$. \qed
 \end{proof}
 }
 
{However, one might think that the pair $(g_\bot, f_\bot)$ is also an element of $\mathcal{F}_N$, but this is not true in general as the following example illustrates.}

\begin{example}\label{ex:botFn}
Let us consider the frame $([0,1],[0,1],[0,1], \leq, \leq, \leq, \&_{\text{\L}})$, where $\&_{\text{\L}}$ is the \L ukasiewicz conjunctor defined for every $x,y\in [0,1]$ as:
$$ x\&_{\text{\L}}y = \max\{0, x+y-1\}$$
The corresponding residuated implications $\swarrow^{\text{\L}}, \nwarrow_{\text{\L}} : [0,1]\times [0,1] \rightarrow [0,1]$ are defined  for every $y,z\in [0,1]$ as:
$$ z\swarrow^{\text{\L}} y = z\nwarrow_{\text{\L}} y=  \min\{1, 1-y+z\}$$

\begin{table}[!h]
				\begin{center}
				\begin{tabular}{|c|cc|}
					\hline
					$R$ & $b_1$ & $b_2$ \\ \hline
					$a_1$& 0.5 & 0  \\
					$a_2$ & 0 & 0.75 \\
					\hline
				\end{tabular}
			\end{center}
		\caption{Fuzzy relation $R$ of Example~\ref{ex:botFn}.}\label{tabla:botFn}
	\end{table}

In addition, we consider the context $(A,B,R)$ where $A=\{a_1,a_2\}$, $B=\{b_1,b_2\}$ and the fuzzy relation $R$ is given in Table~\ref{tabla:botFn}. Now, in order to show that $(g_\bot, f_\bot)$ does not belong to $\mathcal{F}_N$, it is sufficient to verify that $g_\bot^{\upN} \neq f_\bot$. Thus, we apply the operator ${}^{\upN}$ to $g_\bot$, for every $a\in A$, that is:

 \begin{align*}
	 g_\bot^{\upN} (a_1)  
	&=\inf\{ g_\bot(b)\swarrow^{\text{\L}} R(a_1,b) \mid b\in B \} \\
	&=  \inf\{ 0 \swarrow^{\text{\L}} 0.5,  0\swarrow^{\text{\L}} 0 \} \\
	&= \inf\{0.5, 1\}\\
	& = 0.5\\
	g_\bot^{\upN} (a_2)  
	&=\inf\{ g_\bot(b)\swarrow^{\text{\L}} R(a_2,b) \mid b\in B \} \\
	&=  \inf\{ 0 \swarrow^{\text{\L}} 0,  0\swarrow^{\text{\L}} 0.75 \} \\
	& = \inf\{1, 0.25\}\\
	& = 0.25
\end{align*}
Clearly, $g_\bot^{\upN} 
\neq 
 f_\bot$. Therefore, we have that $(g_\bot, f_\bot)\not\in\mathcal{F}_N$.\qed
\end{example}

 On the other hand, it is known that  the set $\mathcal{C}_N$ has the structure of a complete lattice, as we commented in the previous {section}. The following result proves that the fuzzy extension of $\mathcal{C}_N$, that is,  the set $\mathcal{F}_N$ also has the structure of a complete lattice.

\begin{proposition}
	Let  $(L_1, L_2, P,\&, \swarrow, \nwarrow)$ be a frame and let $(A, B,R)$ be a context. Then, the elements in $\mathcal{F}_N$ form a complete lattice with the ordering $(g_1, f_1) \leq (g_2, f_2)$ if and only if $g_1\preceq_2 g_2$, or equivalently, if and only if $f_1 \preceq_1 f_2$. 
\end{proposition}
\begin{proof}
	Let us consider a family of pairs $\{(g_i, f_i)\}_{i\in I}\subseteq\mathcal{F}_N$, with $I$ an index set. Now, we will prove $(\bigwedge_{i\in I} g_i, \bigwedge_{i\in I} f_i) \in \mathcal{F}_N$.

	\begin{align*}
	\left( \bigwedge_{i\in I}g_i \right)^{\upN}(a) 
	&= \inf\{  \left(\bigwedge_{i\in I} g_i(b)\right) \swarrow^{} R(a,b)\mid b\in B\}\\ 
	&\overset{(*)}{=} \inf\{  \bigwedge_{i\in I} (g_i(b) \swarrow^{} R(a,b) )\mid b\in B\}\\
	&= \bigwedge_{i\in I} \left( \inf\{  g_i(b) \swarrow^{} R(a,b) \mid b\in B\}\right)\\ 
	&= \bigwedge_{i\in I} g_i^{\upN}(a) \\
	&=  \bigwedge_{i\in I} f_i(a) 
	\end{align*}
	Note that $(*)$ holds by Proposition~\ref{ijar15}.

	Analogously, we obtain that $ (\bigwedge_{i\in I} f_i )^{\downN}=  \bigwedge_{i\in I} g_i$. Moreover, $(\bigwedge_{i\in I} g_i, \bigwedge_{i\in I} f_i)$ clearly is the greatest lower bound of $\{(g_i, f_i)\}_{i\in I}$.
	Therefore,   
$$
\bigwedge_{i\in I} (g_i, f_i) = \left(\bigwedge_{i\in I} g_i, \bigwedge_{i\in I} f_i\right) \in \mathcal{F}_N
$$
	Furthermore,  by Lemma~\ref{lemma:tops}, $\mathcal{F}_N$ has a top element, this is $(g_\top, f_\top) \in \mathcal{F}_N$ 
	Thus, $\mathcal{F}_N$ forms a complete lattice.\qed
\end{proof}

	Although we have seen that the elements of $\mathcal{C}_N$ and $\mathcal{F}_N$ have the same algebraic structure, the  properties satisfied by the elements of $\mathcal{C}_N$ and $\mathcal{F}_N$ have remarkable differences. For example, in the classical case,  the non-existence of independent blocks entails that the set $\mathcal{C}_N$ only contains the trivial pairs. However, in the fuzzy case, even  when there are no independent subcontexts  the set $\mathcal{F}_N$ may have more elements. This fact is illustrated in the following example.

\begin{example}\label{ex:nordenfg}
		{Consider} the {multi-adjoint} frame 
		$([0, 1]_4, [0, 1]_4, [0, 1]_4, \leq, \leq, \leq, \&^*_{\text{P}})$
		where 
		$\&^*_{\text{P}}$ is the discretization of the product conjunctor (Example~\ref{ex.non.commut}).

		The context is given by the set of attributes $A =\{a_1, a_2, a_3\}$, the set of objects $B = \{b_1, b_2, b_3\}$ and the relation $R_1\colon A\times B \rightarrow [0,1]_4$ shown on the left side of Figure~\ref{ex:fig}.

		\begin{figure}[!h]
			\begin{minipage}{0.55\textwidth}
				\begin{center}
					\begin{tabular}{|c|ccc|}
						\hline
						$R_1$ & $b_1$ & $b_2$ & $b_3$ \\ \hline
						$a_1$& 0.5 & 0.5 & 1 \\
						$a_2$ & 0.25 & 1 & 0\\
						$a_3$ & 0 & 0.75 & 0.25\\
						\hline
					\end{tabular}
				\end{center}
			\end{minipage}
			\begin{minipage}{0.35\textwidth}
				\includegraphics[scale=0.12]{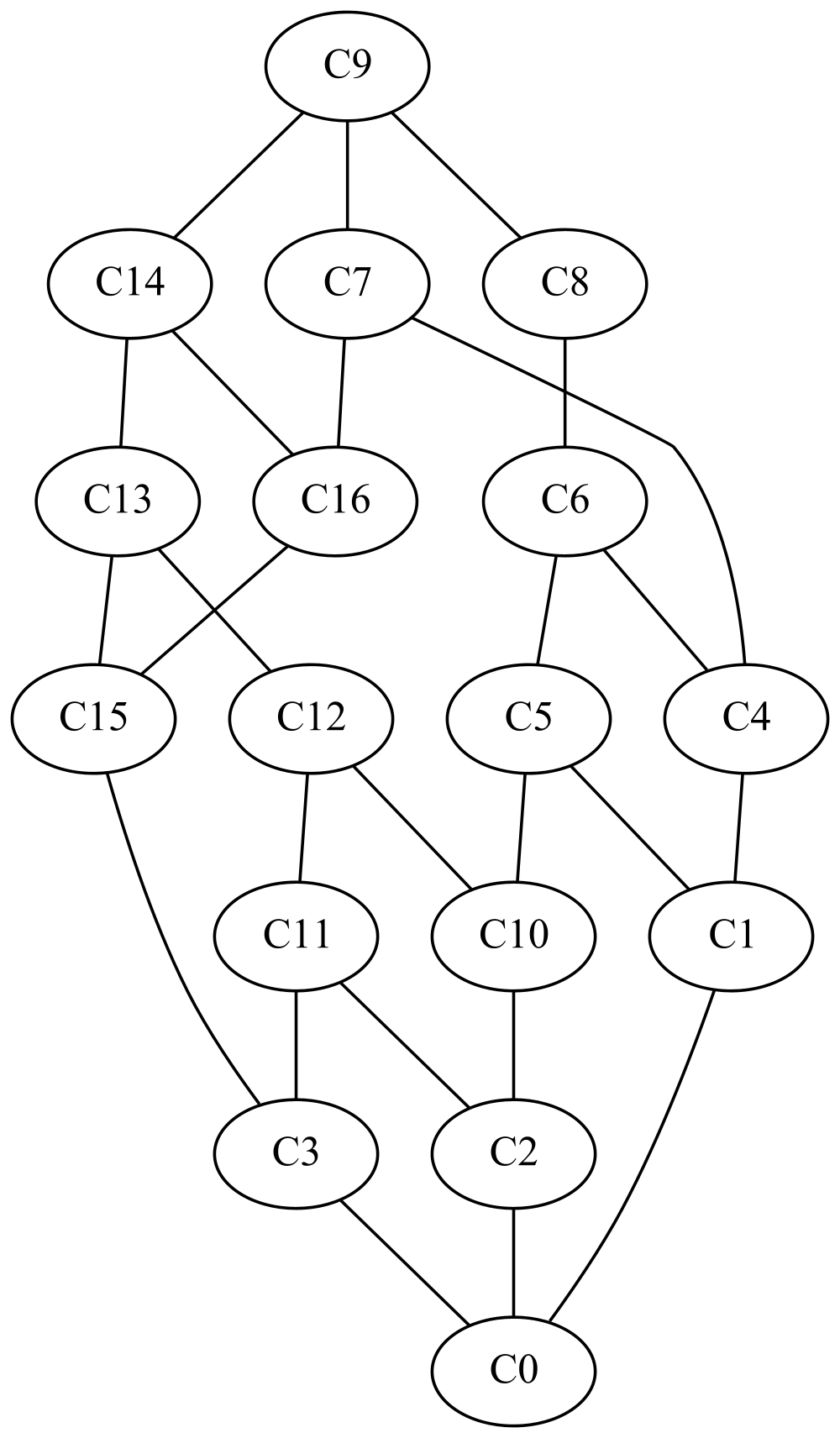}
			\end{minipage}
			\caption{Fuzzy relation $R_1$ and concept lattice of Example~\ref{ex:nordenfg}.}\label{ex:fig}
			\end{figure}

			We can observe that  non-independent  blocks  of concepts exist, as it can be checked in the concept lattice shown in Figure~\ref{ex:fig}.  In this case, the set $\mathcal{F}_N$  contains several elements, despite not containing independent subcontexts. For instance, the pair $(g, f) = ( \{ b_1/1, b_2/0.5, b_3/0.5 \}, \{ a_1/0.5, a_2/0.5, a_3/0.5 \})$  satisfies that $g^{\upN} = f$ and $f^{\downN} = g$, that is, it is an element of $\mathcal{F}_N$.	
		Consequently, the existence of pairs in $\mathcal{F}_N$ does not determine independent blocks of concepts, that is, independent subcontexts.

Notice {that the} fuzzy relation $R_1$ is not normalized because the first  row has all the values different from zero. 
Nevertheless, if the fuzzy relation of the context {were} normalized, it {would be} possible to find these independent blocks of concepts, as it happens in the context $(A, B, R_2)$ and it is shown in  Figure~\ref{ex:fig2}. 	

					\begin{figure}[!h]
			\begin{minipage}{0.55\textwidth}
				\begin{center}
					\begin{tabular}{|c|ccc|}
						\hline
						$R_2$ & $b_1$ & $b_2$ & $b_3$ \\ \hline
						$a_1$& 0.5 & 0 & 1 \\
						$a_2$ & 0 & 0.5 & 0\\
						$a_3$ & 0.75 & 0 & 0.25\\
						\hline
					\end{tabular}
				\end{center}
			\end{minipage}
			\begin{minipage}{0.35\textwidth}
				\includegraphics[scale=0.1]{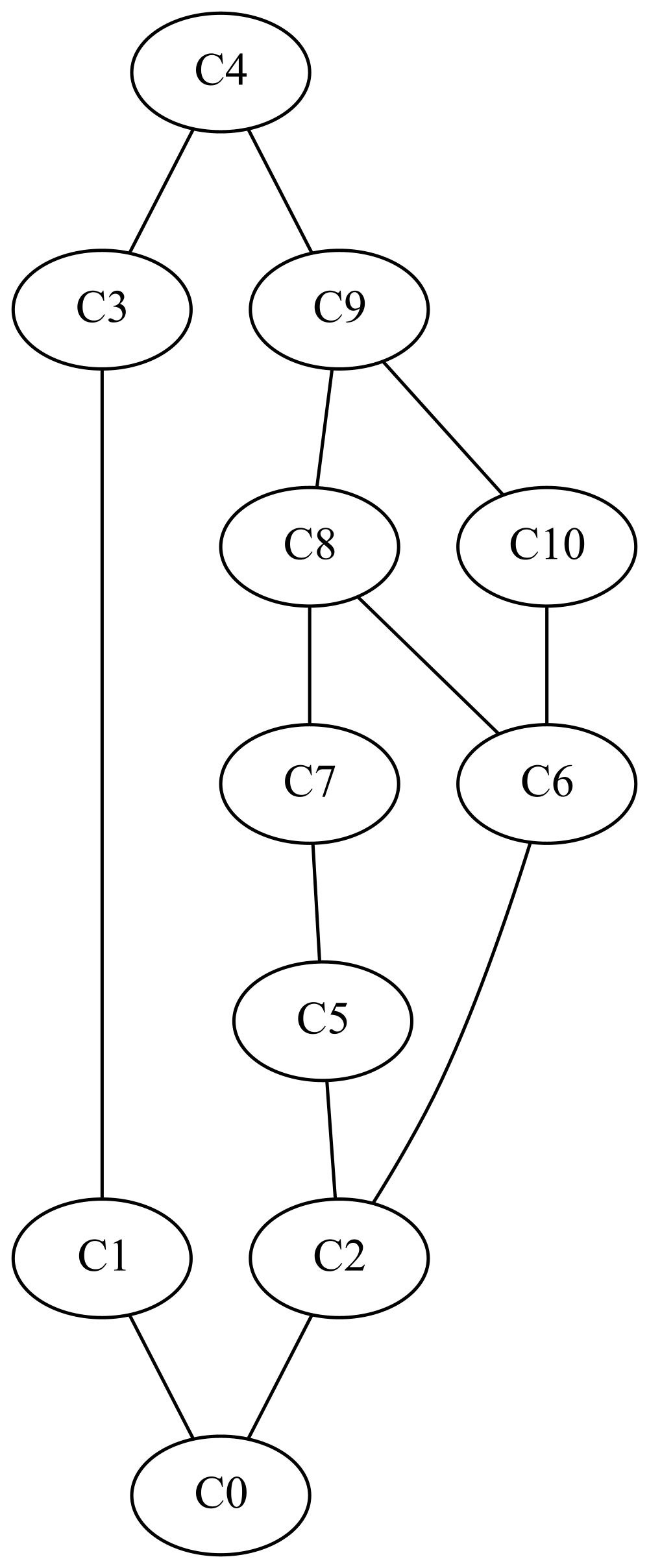}
			\end{minipage}
			\caption{Fuzzy relation $R_2$ and concept lattice of Example~\ref{ex:nordenfg}.}\label{ex:fig2}
		\end{figure}

		\begin{table}[!h]
			\begin{center}
				\begin{tabular}{l}
					$(g_\bot, f_\bot) = (\{ b_1/0, b_2/0, b_3/0\}, \{ a_1/0, a_2/0, a_3/0 \})$ \\
					$(g_1, f_1) = (\{ b_1/0, b_2/1 b_3/0\}, \{ a_1/0, a_2/1, a_3/0 \})$ \\
					$(g_2, f_2) = (\{ b_1/0.25, b_2/0, b_3/0.25 \}, \{ a_1/0.25, a_2/0, a_3/0.25 \})$  \\
					$(g_3, f_3) = (\{ b_1/0.25, b_2/0, b_3/0.5 \}, \{ a_1/0.5, a_2/0, a_3/0.25 \})$ \\
					$(g_4, f_4) = (\{ b_1/0.5, b_2/0, b_3/0.25 \}, \{ a_1/0.25, a_2/0, a_3/0.5 \})$ \\
					$(g_5, f_5) = (\{ b_1/0.5, b_2/0, b_3/1 \}, \{ a_1/1, a_2/0, a_3/0.5 \})$ \\
					$(g_{6}, f_{6}) = (\{ b_1/0.5, b_2/1, b_3/0.25 \}, \{ a_1/0.25, a_2/1, a_3/0.5 \})$ \\
					$(g_{7}, f_{7}) = (\{ b_1/0.5, b_2/1, b_3/0.5 \}, \{ a_1/0.5, a_2/1, a_3/0.5 \})$ \\
					$(g_{8}, f_{7}) = (\{ b_1/0.5, b_2/1, b_3/0.75 \}, \{ a_1/0.75, a_2/1, a_3/0.5 \})$ \\
					$(g_{9}, f_{9}) = (\{ b_1/0.5, b_2/1, b_3/1 \}, \{ a_1/1, a_2/1, a_3/0.5 \})$ \\
					$(g_{10}, f_{10}) = (\{ b_1/1, b_2/0, b_3/1 \}, \{ a_1/1, a_2/0, a_3/1\})$ \\
					$(g_\top, f_\top) = (\{ b_1/1, b_2/1, b_3/1 \}, \{ a_1/1, a_2/1, a_3/1 \})$ 
				\end{tabular}
			\end{center}
			\caption{Some pairs of $\mathcal{F}_N$ in the context $(A, B,R_2)$ of Example~\ref{ex:nordenfg}.}\label{ex:fn2}
		\end{table}	

The set $\mathcal{F}_N$ in the context $(A, B,R_2)$ consists of 20 pairs, some of them are listed in Table~\ref{ex:fn2}.  
For instance, the pair $(g_{10},f_{10})$ 
characterizes the block delimited by the concepts $C_2$ and $C_9$ shown in the concept lattice in Figure~\ref{ex:fig2}. However, not all of them provide independent blocks, which raises the need to carry out the study  of more properties of the set $\mathcal{F}_N$. \qed

\end{example}

The following result shows the relation between the possibility and necessity operators applied to the same fuzzy set of objects $g$ of a considered pair $(g,f)\in\mathcal{F}_N$.

\begin{proposition}\label{prop:fp1}
	Given a frame $(L_1, L_2, P,\&, \swarrow, \nwarrow)$, a context $(A, B, R)$ and a pair $(g,f)\in\mathcal{F}_N$, we have that ${g}^{\upPi}(a) \preceq_1 {g}^{\upN}(a)$, for all $a\in A$.
\end{proposition} 
\begin{proof}
	Given $(g,f)\in\mathcal{F}_{N}$, and $b\in B$, we have that the expression:
	
	$$
	{g}(b) = {g}^{\upN\downN}(b) = \inf\{\inf\{g(b')\swarrow R(a',b') \mid b'\in B\}\nwarrow R(a',b) \mid a'\in A\}
	$$
	Therefore, by Proposition~\ref{ijar15}, we have that $g(b) = \inf\{(g(b')\swarrow R(a',b'))\nwarrow R(a',b)\mid a'\in A, b'\in B\}$. 
	Thus, given  $a'\in A$ and $b'\in B$ 
$$
g(b) \preceq_1  (g(b')\swarrow R(a',b'))\nwarrow R(a',b)
$$
 and applying the adjoint property (Equivalence~(\ref{ap})) we obtain that 
$$
 R(a',b)\& g(b) \preceq_1 g(b')\swarrow R(a',b')
$$ 
Now, we can apply the supremum on the left argument of the previous inequality and infimum in the right argument, obtaining that: 

$$
 \sup\{R(a',b)\& g(b)\mid b\in B\} \preceq_1 \inf\{g(b')\swarrow R(a',b') \mid b'\in B\}
 $$ 
 
 Therefore, ${g}^{\upPi}(a') \preceq_1 {g}^{\upN}(a'),$ for all $a' \in A$.\qed

\end{proof}

The following property relates the closure operator that arises from the composition of the mappings of the pair $(^{\upPi}, ^{\downN})$ to the composition of the necessity operators $^{\upN}$ and $^{\downN}$.

\begin{proposition}\label{prop:fp2}
	Let $(L_1, L_2, P,\&, \swarrow, \nwarrow)$ be a frame, $(A, B, R)$ be a context and $(g,f)\in\mathcal{F}_N$. Then, ${g}^{\upPi\downN}${$= g$, and therefore $\langle g, {g}^{\upPi}\rangle\in \mathcal{M}_{\pi N}$.}
\end{proposition} 
\begin{proof}
		By Proposition~\ref{prop:fp1}, we have that ${g}^{\upPi} \preceq_1 {g}^{\upN}$. If we apply the operator $^{\downN}$, we have that ${g}^{\upPi\downN} \preceq_2 {g}^{\upN\downN}$, since $^{\downN}$ is order preserving. In addition, we have that $g \preceq_2  {g}^{\upPi\downN}$, since the composition of the mappings in the pair $(^{\upPi}, ^{\downN})$ is a closure operator and $ {g}^{\upN\downN} = g$, by hypothesis. Therefore, we obtain that
\[
g \preceq_2  {g}^{\upPi\downN}\preceq_2 {g}^{\upN\downN} = g
\]

Thus, we can conclude that the pair $\langle g, {g}^{\upPi}\rangle$ is a concept of the property-oriented concept lattice.
\qed
\end{proof}

Hence, the previous result asserts that the pair $\langle g, {g}^{\upPi}\rangle$ is always a property-oriented concept, whenever the pair $(g, f)$ is in $\mathcal{F}_N$, as in the Boolean  case. Nevertheless, the extension of Proposition~\ref{prop:cp1} to the fuzzy framework does not hold, that is, the pair $( g, {g}^{\up})$ is not a concept in general, as the following example shows.

\begin{example}
Considering again the context $(A,B,R_2)$ of Example~\ref{ex:nordenfg}, we focus on the pair of $\mathcal{F}_N$:
\[(g_2, f_2) = (\{ b_1/0.25, b_2/0, b_3/0.25 \}, \{ a_1/0.25, a_2/0, a_3/0.25 \})\]
Applying the operator ${}^\up$ to the fuzzy subset $g_2$, we obtain that:
\[(g_2, g_2^{\up})=(\{ b_1/0.25, b_2/0, b_3/0.25 \}, \{ a_1/1, a_2/0, a_3/1 \})\]
However,  although $g_2^{\up}$ is different from $g_{\bot}$, if we compute $g_2^{\up\down}$ we have that
$g_2^{\up\down}=\{ b_1/0.5, b_2/0, b_3/0.5 \}$.
Therefore, we can conclude that $g_2^{\up\down}\neq g_2$ and, as a consequence, the pair $( g_2, {g_2}^{\up})$ is not a concept.\qed
\end{example}

As a consequence,  the pairs $(g,f)$ in $ \mathcal F_N$ do not provide in general a concept   in $\mathcal M$ either with extent  $g$ or intent $f$. Therefore,  Proposition~\ref{prop:cp2} and Proposition~\ref{prop:cp3} cannot be extended to the multi-adjoint frame either.  

Due to these properties given in the Boolean case are not satisfied in the fuzzy framework, it is necessary to set some assumptions which allow us to obtain good properties related to the factorization of fuzzy contexts.

\subsection{$\top$-normalized contexts}

From now on, we will consider a framework for which interesting results about the factorization of contexts in the fuzzy case can be obtained. Specifically, we will consider the following kind of formal context.

\begin{definition}\label{def.tnomal} 
A normalized context $(A, B, R)$,  satisfying that for every $a\in A$, there exist $b_a\in B$ such that $R(a, b_a)=\top_P$ is called \emph{$\top$-normalized context}. 
\end{definition}

Notice that, the assumption on $R$  means that the fuzzy subsets associated with every row of $R$ be a normal fuzzy subset, in this case, we will say that $R$ is \emph{normal by rows}. This fact  is not restrictive because, for example, if we are handling  real numbers, it is usual to normalize (divide by the greatest element of the row) in order to obtain values in the unit interval, if they are not in this interval, or to extend them, if they are mainly concentrated in a part of the unit interval.

Moreover, we will consider  the frame $([0, 1], \leq, \&_{\text{G}})$ where $\&_{\text{G}}$ is the G\"odel conjunctor, that is, the operator defined for every $x,y\in [0,1]$ as:
$$
x \&_{\text{G}} y = \min\{ x,y\}
$$
The corresponding residuated implications $\swarrow^{\text{G}}, \nwarrow_{\text{G}} : [0,1]\times [0,1] \rightarrow [0,1]$ are defined  for every $y,z\in [0,1]$ as:
$$ z\swarrow^{\text{G}} y = z\nwarrow_{\text{G}} y=  \displaystyle\left\{\begin{array}{ll}
		1 & \mbox{ if } y\leq z\\
		\displaystyle z & \mbox{ otherwise }
		\end{array}\right.
$$
	
Under these new hypotheses, the opposite inequality to the one given in Proposition~\ref{prop:fp1} also holds, as the following result shows.
\begin{proposition}\label{prop:fp3}
	Given the frame $([0,1],\leq, \&_\text{G})$, a $\top$-normalized context $(A, B, R)$ and a pair $(g,f)\in\mathcal{F}_N$, 
	it is satisfied that ${g}^{\upN} \leq {g}^{\upPi}$. 
\end{proposition}
\begin{proof}
Let us consider an arbitrary attribute $a\in A$ and a pair $(g,f)\in\mathcal{F}_N$. 
By definition of the necessity operator, we have that:
\[g^{\upN}(a) = \inf\{g(b)\swarrow^\text{G} R(a,b) \mid b\in B\}\]
 Furthermore, 
taking into account the definition of the G\"odel residuated implication, the previous equality can be written as follows:
\[
g^{\upN}(a) = \inf\{g(b) \mid b\in B \mbox{ and } R(a,b)\nleq g(b)\} 
\]
Notice that, if the set $\{g(b) \mid b\in B \mbox{ and } R(a,b)\nleq g(b)\}$ is the empty set, then its infimum is the maximum of the frame, that is, 1 in this case.

In addition, since the context is a $\top$-normalized context,  there exists at least one object $b'$ such that $R(a, b')=1$. Now, we have to distinguish two cases:
\begin{itemize}
\item  If $R(a,b')\leq g(b')$, then $R(a,b')\&_\text{G} g(b') =\min\{R(a,b'), g(b')\}= 1$ and,  as a consequence, $1= \sup\{ R(a,b)\&_\text{G} g(b) \mid b\in B \}=g^{\upPi}(a)$. Therefore,  the inequality $ g^{\upN}(a)\leq g^{\upPi}(a)$ holds.

\item Otherwise, if $R(a,b')\nleq g(b')$, then $g(b') \in \{g(b) \mid b\in B \mbox{ and } R(a,b)\nleq g(b)\} $. Therefore, we have that:\[ g^{\upN}(a)=\inf\{g(b) \mid b\in B \mbox{ and } R(a,b)\nleq g(b)\} \leq g(b')\]

On the other hand, since $R(a,b')\nleq g(b')$ the following chain holds:

 \[g(b')=\min\{R(a,b'), g(b')\}
  \leq \sup\{ R(a,b)\&_\text{G} g(b) \mid b\in B \}=g^{\upPi}(a) \]

Therefore, from both inequalities we have that $g^{\upN}(a) \leq g^{\upPi}(a)$, for all $a\in A$.
\end{itemize}
Consequently, from the previous items we conclude that $g^{\upN}\leq g^{\upPi}$.
\qed
\end{proof}

Furthermore, we can also establish a relationship between the derivation operators and the possibility operators, as the following result states.
\begin{proposition}\label{prop:fp4}
	Given the frame $([0,1],\leq, \&_\text{G})$, a $\top$-normalized context $(A, B, R)$ and a pair $(g,f)\in\mathcal{F}_N$ satisfying that for every $a\in A$, there exists $b\in B$  such that $g(b)\nleq R(a,b)$, 
	then the inequality ${g}^{\up} \leq {g}^{\upPi}$ holds. 
\end{proposition}
\begin{proof}
	Let us consider an arbitrary attribute $a\in A$ and a pair $(g,f)\in\mathcal{F}_N$ satisfying $g(b')\nleq R(a,b')$, with $b'\in B$.
	By definition, we have that  $g^{\up}(a) = \inf\{R(a,b)\swarrow_\text{G} g(b) \mid b\in B\}$.  Taking into account the definition of the G\"odel residuated implication, the previous equality can be written as:
	\[g^{\up}(a) = \inf\{R(a,b) \mid b\in B \mbox{ and } g(b)\nleq R(a,b)\} \]
	Notice that the set $\{R(a,b) \mid b\in B \mbox{ and } g(b)\nleq R(a,b)\}$ is not empty, since $R(a,b')$ belongs to it. Therefore, we have that:
\[g^{\up}(a) = \inf\{R(a,b) \mid b\in B \mbox{ and } g(b)\nleq R(a,b)\}\leq R(a,b')\]	
	
In addition, we know that $b'\in B$ satisfies that $g(b')\nleq R(a,b')$, from which the following chain is deduced: 
\[R(a,b')=\min\{R(a,b'),g(b')\}
 \leq \sup\{ R(a,b)\&_\text{G} g(b) \mid b\in B \}=g^{\upPi}(a)\]

Consequently, from both inequalities we obtain that $g^{\up}  \leq g^{\upPi} $.\qed

\end{proof}

The following result shows another particular case in which the inequality given in the previous result is also satisfied.

\begin{proposition}\label{prop:fp4'}
	Given the frame $([0,1],\leq, \&_\text{G})$, a $\top$-normalized context $(A, B, R)$ and a pair $(g,f)\in\mathcal{F}_N$, if for any $a\in A$ there exists $b\in B$ such that $R(a,b) = g(b) = 1$, then $g^{\upPi}=g_\top$ and, in particular ${g}^{\up} \leq {g}^{\upPi}$ is satisfied.
\end{proposition}

\begin{proof}
We consider $a\in A$ such that there exists $b'\in B$ such that $R(a,b') = g(b') = 1$, then  $g^{\upPi}(a)=\sup\{ R(a,b)\&_\text{G} g(b) \mid b\in B \}=1$ and the inequality $g^{\up}(a) \leq g^{\upPi}(a)$ trivially holds.\qed
\end{proof}

As we have already seen in Example~\ref{ex:nordenfg}, the existence of pairs of $\mathcal{F}_N$ does not imply that these pairs determine independent blocks of concepts. In what follows, we will show  under what conditions these pairs provide intervals or blocks of concepts (that could not be independent). 	
 Taking into account Proposition~\ref{prop:fp3} and Proposition~\ref{prop:fp4}, we can find a relationship between  the two concepts $\langle g^{\up\down}, {g}^{\up}\rangle$ and $\langle f^{\down}, {f}^{\down\up}\rangle$ obtained from pairs in $\mathcal{F}_N$. 

\begin{proposition}\label{prop:fp5}
	Let $(A, B, R)$ be $\top$-normalized context, the frame $([0,1],\leq, \&_\text{G})$ and a pair $(g,f)\in \mathcal{F}_N$ satisfying that for every $a\in A$ there exists $b\in B$ such that  $g(b)\nleq R(a,b)$. Then, the inequality ${f}^{\down}\leq {g}^{\up\down}$ holds and, therefore, $\langle f^{\down}, {f}^{\down\up}\rangle \preceq \langle g^{\up\down}, {g}^{\up}\rangle$. 
\end{proposition}
\begin{proof}
Applying the operator $^\down$ to the inequality  given by Proposition~\ref{prop:fp4}, we obtain that $g^{\upPi\down} \leq g^{\up\down}$, since the operator $^\down$ is order-reversing. Moreover, taking into consideration Proposition~\ref{prop:fp1} and Proposition~\ref{prop:fp3}, the equality $g^{\upPi} = g^{\upN}$ holds. Therefore, replacing $g^{\upPi}$ by $g^{\upN}$ in the first inequality, we have that $f^{\down}=g^{\upN\down} =  g^{\upPi\down} \leq g^{\up\down}$. 
\qed
\end{proof}

Furthermore, the following result shows another situation in which the inequality shown in the previous proposition is also satisfied.

\begin{proposition}\label{prop:fp5'}
	Let $(A, B, R)$ be a $\top$-normalized context, the frame $([0,1],\leq, \&_\text{G})$ and a pair $(g,f)\in \mathcal{F}_N$. Given $a\in A$ if there exists $b\in B$ satisfying that $R(a,b) = g(b) = 1$, the inequality ${f}^{\down} \leq {g}^{\up\down}$ is satisfied, and so $\langle f^{\down}, {f}^{\down\up}\rangle \preceq \langle g^{\up\down}, {g}^{\up}\rangle$.
\end{proposition}

\begin{proof}
Similarly to the proof given in Proposition~\ref{prop:fp5}, but taking into account Proposition~\ref{prop:fp4'} instead of Proposition~\ref{prop:fp4}.\qed
\end{proof}


Notice that, if we consider in Definition~\ref{def.tnomal} that the fuzzy subsets associated with the  columns of $R$ be normal fuzzy subsets, that is, $R$ is normal by columns instead of by rows, then we obtain Propositions~\ref{prop:fp3},~\ref{prop:fp4}, and~\ref{prop:fp4'} related to the mapping $f$   of the pair $(g,f)\in\mathcal{F}_N$. As a consequence, Propositions~\ref{prop:fp5} and~\ref{prop:fp5'} also hold, if $R$ is normal by columns.

In the following example,  we  illustrate all the previous results. 

\begin{example}\label{ex:intervals}

Given the frame $([0, 1], \leq, \&_{\text{G}})$, and the context $(A, B, R_1)$ composed of the set of attributes $A =\{a_1, a_2, a_3\}$, the set of objects $B = \{b_1, b_2, b_3\}$ and the relation $R_1\colon A\times B \rightarrow [0,1]$ defined in Table~\ref{tabla:r_no-norm}.
	
	\begin{table}[!h]
				\begin{center}
				\begin{tabular}{|c|ccc|}
					\hline
					$R_1$ & $b_1$ & $b_2$ & $b_3$ \\ \hline
					$a_1$& 0.5 & 0.25 & 0 \\
					$a_2$ & 0.5 & 1.0 & 0\\
					$a_3$ & 0 & 0 & 0.75\\
					\hline
				\end{tabular}
			\end{center}
		\caption{Fuzzy relation $R_1$ of Example~\ref{ex:intervals}.}\label{tabla:r_no-norm}
	\end{table}
	As we can observe, the fuzzy relation $R_1$ is a normalized context, but it is not a normal fuzzy relation, therefore it is not a $\top$-normalized context. This fact gives rise to not satisfying the hypotheses  of Proposition~\ref{prop:fp3}. For instance, the pair $(g,f) = (\{ b_1/1, b_2/0.5, b_3/0 \}, \{ a_1/1, a_2/0.5, a_3/0\})$ belongs to $\mathcal{F}_N$ and, if we apply the possibility operator to the fuzzy subset $g$, we obtain that:
	\[g^{\upN}= f = \{a_1/1, a_2/0.5, a_3/0\} \nleq \{a_1/0.5, a_2/0.5, a_3/0\}=g^{\upPi}\]

\begin{table}[!h]
	\begin{center}
		\begin{tabular}{|c|ccc|}
			\hline
			$R_2$ & $b_1$ & $b_2$ & $b_3$ \\ \hline
			$a_1$& 1 & 0.25 & 0 \\
			$a_2$ & 0.5 & 1 & 0\\
			$a_3$ & 0 & 0 & 1\\
			\hline
		\end{tabular}
	\end{center}
	\caption{Fuzzy relation $R_2$ of Example~\ref{ex:intervals}.}\label{tabla:r_norm}
\end{table}
Now, we consider a $\top$-normalized context $(A, B, R_2)$ in the same frame and where the fuzzy relation $R_2$ is given in Table~\ref{tabla:r_norm}. 
We can compute all pairs from $\mathcal{F}_N$, some of them are listed in Table~\ref{ex:fn3}. The associated concept lattice of the context along with the list of concepts are shown in Figure~\ref{ex:fig_intervalos}.

	\begin{table}[!h]
	\begin{center}
		\begin{tabular}{l}
			$(g_\bot, f_\bot) = (\{ b_1/0, b_2/0, b_3/0\}, \{ a_1/0, a_2/0, a_3/0 \})$ \\
			$(g_1, f_1) = (\{ b_1/0, b_2/0, b_3/0.5\}, \{ a_1/0, a_2/0, a_3/0.5 \})$ \\
			$(g_2, f_2) = (\{ b_1/0.75, b_2/0.5, b_3/0 \}, \{ a_1/0.75, a_2/0.5, a_3/0 \})$  \\
			$(g_3, f_3) = (\{ b_1/1, b_2/0.75, b_3/0 \}, \{ a_1/1, a_2/0.75, a_3/0 \})$ \\
			$(g_4, f_4) = (\{ b_1/0.75, b_2/0.5, b_3/0.5 \}, \{ a_1/0.75, a_2/0.5, a_3/0.5 \})$\\
			$(g_5, f_5) = ( \{ b_1/0, b_2/0, b_3/0.5 \}, \{ a_1/0, a_2/0, a_3/0.5 \})$ \\
			$(g_\top, f_\top) = (\{ b_1/1, b_2/1, b_3/1 \}, \{ a_1/1, a_2/1, a_3/1 \})$ 
		\end{tabular}
	\end{center}
	\caption{Some pairs of $\mathcal{F}_N$ in the context $(A, B,R_2)$ of Example~\ref{ex:intervals}.}\label{ex:fn3}
\end{table}	

Taking into account $(g_1, f_1)$, it easy to check that $g_1(b)\leq R_2(a_3, b)$ for every $b\in B$. Therefore, we are not under the hypothesis of Proposition~\ref{prop:fp4}. Thus, we can observe that Proposition ~\ref{prop:fp4} does not hold for $a_3\in A$, that is,

\[g_1^{\up}(a_3)= 1 \nleq 0.75=g_1^{\upPi}(a_3)\]

However, considering the pair $(g_2,f_2)$ which satisfies the hypothesis of Proposition~\ref{prop:fp4}, we obtain that  Proposition~\ref{prop:fp5} also holds for the pair $(g_2,f_2)$, that is, $\langle f_2^\down, f_2^{\down\up}\rangle\preceq \langle g_2^{\up\down}, g_2^{\up}\rangle$ as we show below:
\[ f_2^\down = \{ b_1/1, b_2/0.25, b_3/0 \} \leq \{ b_1/1, b_2/1, b_3/0\} = g_2^{\up\down}\]

Notice that these concepts correspond to $C_3$ and $C_5$, respectively, providing an interval of concepts in the concept lattice.

\begin{figure}[!h]
				\begin{minipage}{0.7\textwidth}
						\begin{tabular}{l}
								$C_0 =\langle\{b_1/0, b_2/0, b_3/0\}, \{a_1/1.0 , a_2/1.0 , a_3/1.0 \}\rangle$\\
								$C_1 = \langle\{b_1/0.5, b_2/0.25, b_3/0 \} , \{a_1/1.0 , a_2/1.0, a_3/0 \}\rangle$\\
								$C_2 = \langle\{b_1/0, b_2/0, b_3/1 \} , \{a_1/0 , a_2/0 , a_3/1 \}\rangle$\\
								$C_3 = \langle\{b_1/1, b_2/0.25, b_3/0 \} , \{a_1/1.0 , a_2/0.5, a_3/0 \}\rangle$\\
								$C_4 = \langle\{b_1/0.5, b_2/1, b_3/0 \} , \{a_1/0.25 , a_2/1, a_3/0 \}\rangle$\\
								$C_5 = \langle\{b_1/1, b_2/1, b_3/0  \} , \{a_1/0.25 , a_2/0.5, a_3/0 \}\rangle$\\
								$C_6 =\langle\{b_1/1, b_2/1, b_3/1\} , \{a_1/0, a_2/0, a_3/0 \}\rangle$
							\end{tabular}
					\end{minipage}
				\begin{minipage}{0.2\textwidth}
						\includegraphics[scale=0.1]{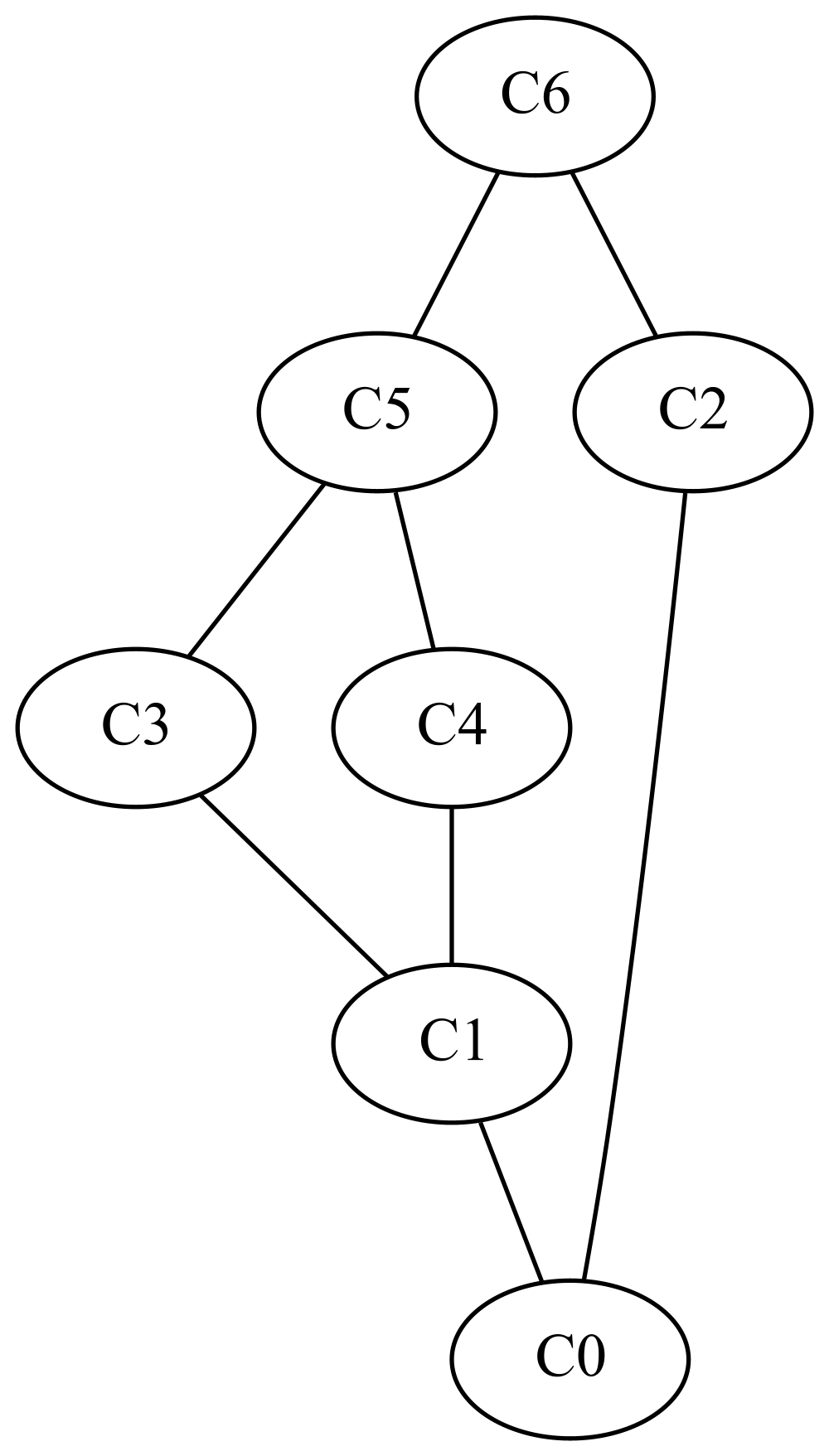}
					\end{minipage}
				\caption{List of concepts and concept lattice associated with the context $(A, B, R_2)$ of Example~\ref{ex:intervals}.}\label{ex:fig_intervalos}
				\end{figure}
				
On the other hand, according to Proposition~\ref{prop:fp5'}, we can verify that $(g_3,f_3)$ satisfies that $g_3(b_1) = R(a_1,b_1) = 1$ and, in this occasion, we obtain a block of concepts delimited by $C_1$ and $C_5$, that is, 
\[f_3^\down = \{ b_1/1, b_2/1, b_3/0 \}\leq \{ b_1/0.5, b_2/0.25, b_3/0\} = g_3^{\up\down}\]

Similarly, the block which is  only composed of the concept $C_2$ is determined by the pair $(g_5,f_5)$, since $\langle f_5^\down, f_5^{\down\up}\rangle= \langle g_5^{\up\down}, g_5^{\up}\rangle = C_2$.

Notice that there are some pairs of $\mathcal{F}_N$ that  provide an interval which represents the whole concept lattice. For instance, Proposition~\ref{prop:fp5} holds for the pair $(g_4,f_4)$ and this pair determines the interval $C_0 = \langle f_4^\down, f_4^{\down\up}\rangle \preceq \langle g_4^{\up\down}, g_4^{\up}\rangle = C_6$.

Note that, the  corresponding concept lattice of the latter context in this example contains  two independent blocks of concepts. Thanks to the results obtained in this paper, we have been able to establish intervals of concepts within the concept lattice, from the pairs in $\mathcal{F}_N$. One of these intervals delimits one independent block. Specifically,  this block has been determined from the pair $(g_3,f_3)$ in $\mathcal{F}_N$. Therefore, the following step will be to discover what elements in $\mathcal{F}_N$ can characterize the independent blocks of concepts in a concept lattice.\qed

\end{example}

Consequently, we can find different intervals or blocks of concepts by means of pairs of $\mathcal{F}_N$ under certain assumptions,  as we have proven in the previous results and  illustrated in Example~\ref{ex:intervals}. Notice that the introduced consequences  can be easily proved considering in the frame the Product conjunctor instead of the G\"odel conjunctor.

\section{Conclusions and future work}\label{conclusion}

In this paper, we have continued with the study initiated in~\cite{aragonCCIS22}, on the necessity operator when it is used to factorize a formal context into independent subcontexts, which can provide particular consequences/information from the whole dataset and from which the original context can be recovered. We have presented several properties that are satisfied by blocks of concepts associated with independent subcontexts in classical FCA. For example, we have identified the bounds associated with each block of concepts, that is, we have identified the interval of the concept lattice associated with each element in $\mathcal C_N$. Furthermore, we have shown how some of the properties in the classical case are extended to the fuzzy framework. We have also provided   a specific fuzzy frame in which we have {obtained} additional properties, such as, the determination of intervals of concepts by means of the elements $\mathcal F_N$. 

In the near future, we are interested in providing a mechanism to factorize formal contexts in the multi-adjoint framework. Moreover, we will continue with the study of more properties related to independent subcontexts in both, the classical and the fuzzy framework. In particular, we will study under what conditions the intervals of concepts, determined by means of elements in $\mathcal F_N$, characterize independent subcontexts. We will also consider alternative ways to address the decomposition of formal contexts into smaller subcontexts. Furthermore,  factorization techniques can be very useful to analyze the information contained in real databases, as it has been shown in~\cite{BELOHLAVEK2015,Chi2023}. For that reason, we will apply the obtained results to real datasets, as the ones given in~\cite{Ene2008,Myllykangas2006}, in order to complement the considered factorization techniques and, in this way  obtaining more optimal mechanisms.

\end{document}